\documentclass[11pt]{article}

\usepackage[letterpaper,margin=0.95in]{geometry}
\usepackage[T1]{fontenc}
\usepackage[utf8]{inputenc}
\usepackage{lmodern}
\usepackage{microtype}
\usepackage{float}
\usepackage{needspace}

\usepackage{amsmath,amssymb,amsthm,mathtools}
\usepackage{graphicx}
\usepackage{subcaption}
\usepackage{booktabs}
\usepackage{xcolor}
\usepackage[most]{tcolorbox}
\usepackage{hyperref}
\usepackage[numbers,sort&compress]{natbib}
\theoremstyle{definition}
\newtheorem{definition}{Definition}[section]
\usepackage{algorithm}
\usepackage{algpseudocode}
\usepackage{booktabs}
\usepackage{multirow}

\definecolor{tildeBlack}{HTML}{0F0F0F}
\definecolor{tildeMuted}{HTML}{666666}
\definecolor{tildeLight}{HTML}{F2F3F5}
\definecolor{tildeBlue}{HTML}{1A73E8}

\hypersetup{
  colorlinks=true,
  linkcolor=black,
  citecolor=blue,
  urlcolor=black
}

\usepackage{amsthm}
\usepackage{booktabs}
\usepackage{multirow}
\usepackage{siunitx}

\newtheorem{prop}{Proposition}[section]
\newtheorem{theorem}{Theorem}[section]

\usepackage{pifont}

\newcommand{\xmark}{\ding{55}}

\usepackage{abstract} 
\newcommand{\tildelogo}{%
  \includegraphics[height=1.1em]{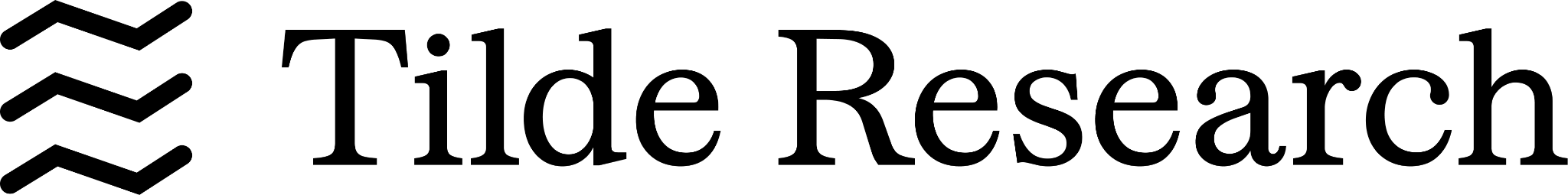}%
}

\newtcolorbox{herobox}{
  enhanced,
  colback=tildeLight,
  colframe=tildeLight,
  boxrule=0pt,
  arc=14pt,
  outer arc=14pt,
  left=20pt,
  right=20pt,
  top=18pt,
  bottom=16pt,
  width=\textwidth
}

\newcommand{\paperdate}{June 26, 2026}
\newcommand{\blogurl}{https://blog.tilderesearch.com/blog/aurora}
\newcommand{\codeurl}{https://github.com/tilde-research/aurora-release}

\newcommand{\auroraheader}{
\noindent
\begin{minipage}[c]{0.5\textwidth}
\tildelogo
\end{minipage}
\begin{minipage}[c]{0.5\textwidth}
\raggedleft{\large \paperdate}
\end{minipage}

\noindent\rule{\textwidth}{0.8pt}

\vspace{1.6em}

{\centering
\LARGE\bfseries
Aurora: A Leverage-Aware Spectral Optimizer\par
}

\vspace{0.8em}

\noindent\rule{\textwidth}{0.8pt}

\vspace{1.0em}

\begin{center}
{\large\bfseries
Alec Dewulf\footnote{Correspondence to alec@tilderesearch.com},
Dhruv Pai,
Li Yang,
Ashley Zhang,
Ben Keigwin
\par}

\vspace{0.8em}

{\large Tilde Research\par}

\vspace{0.8em}

\begin{minipage}{0.72\textwidth}
\raggedright
Blogpost: \url{\blogurl}\\
Code: \url{\codeurl}
\end{minipage}
\end{center}

\vspace{0.6em}
}

\begin{document}

\auroraheader

\begin{abstract}
We show that for tall matrix parameters, like projection matrices in the MLP layers, the Muon update can have row norms that are arbitrarily non-uniform. 
%, which causes large variance in per-neuron learning rates in the MLP layers. 
This can lead to a self-reinforcing feedback loop whereby neurons receive persistently small updates and eventually do not contribute meaningfully to network outputs. This problem is effectively mitigated by an additional row normalization step, but current methods do this in a way that moves the Muon update geometry away from the polar factor of the momentum matrix, which we find is undesirable. We propose Aurora, an optimizer that enforces row-uniformity of matrix parameter updates while respecting Muon's polar factor geometry. Aurora outperforms Muon in our pre-training experiments and, when combined with existing methods, achieves state-of-the-art performance among spectral optimizers on the optimizer track of the \texttt{modded-nanoGPT} speedrun. Additionally, we find that Aurora's empirical gains over Muon scale with the MLP expansion factor, suggesting that Aurora may allow for effective training of very wide MLP layers.
\end{abstract}

\vspace{2.0em}

% ------------------------------------------------------------
% Main paper
% ------------------------------------------------------------

\section{Introduction}
Muon~\cite{jordan2024muon} has become an increasingly popular optimizer for training language models up to frontier scale~\citep{deepseekai2026deepseekv4, kimiteam2026kimi25, singh2026arcee, glm5team2026glm5vibecodingagentic}. The core operation in Muon is an orthogonalization routine that approximately replaces a matrix-valued update with its polar factor, preserving the update's singular vectors while removing its singular value scales.

Muon is part of a growing family of \textit{spectral optimizers} which apply some explicit transformation to the spectrum of the update matrix. Most spectral optimizers use orthogonalization, but recent work has explored the possibility of using softer singular value maps, which shrink large singular values less aggressively \cite{li2026intrinsic}. Other spectral optimizers add conditioning steps, normalization, and adaptive rescaling, with some variants showing significantly improved convergence speed on community research benchmarks like \texttt{modded-nanoGPT}~\citep{modded_nanogpt_2024, chang2026muoneq, si2025adamuonadaptivemuonoptimizer}.

One such example is NorMuon \cite{li2025normuonmakingmuonefficient}, which rescales each row of the polar factor by its inverse RMS norm. This row normalization step was proposed as a mechanism to control variance in the update's row norms for rectangular matrices, which may be large under Muon. We show that this variance can lead to significant under-utilization of parameters in the MLP layers, and that in some cases neurons can become effectively dead. We find that row normalization effectively addresses this issue but at the cost of modifying the geometry of the Muon update, which is undesirable both empirically and under the Muon theoretical formulation.

We propose the \textbf{Aurora} optimizer as a way to control row norm variance in neuron updates without sacrificing precision of the polar factor computation. The Aurora update rule is specifically for the up and gate projection matrices in the MLP layers. %, which we show can receive uneven updates under Muon.
Aurora effectively prevents neuron death in all our experiments and results in faster convergence speeds compared to Muon and other row-normalized spectral optimizers. Additionally, we find that Aurora's empirical gains over Muon scale with the MLP expansion factor, suggesting that Aurora may allow for effective training of networks with unusually wide MLP layers.

%We stress that Aurora is intended to \textit{unify} the strengths of row-normalized spectral optimizers rather than to add to the growing list of variants. We provide extensive documentation of the neuron death phenomenon and a mechanistic analysis of how it can occur in the hope that practitioners will be able to quickly decide whether Aurora is suitable for their setting.

%We present extensive evidence that Aurora achieves high polar factor precision and effectively distributes update magnitude across rows

\paragraph{Contributions.} Our contributions are as follows.
\begin{enumerate}
    \item \textbf{Analysis of normalized spectral optimizers.} We show that post-hoc unit row normalization can move the Muon update geometry far from that of the true polar factor; additionally, we show that row normalization is not required for square or wide parameters under precise orthogonalization. This leads to the design of U-NorMuon, a simple modification of NorMuon that restricts row normalization to tall matrices and removes additional row-norm buffers. U-NorMuon outperforms Muon in our experiments, providing evidence for the hypothesis that row normalization and precise orthogonalization are concurrently desirable.
    %\item \textbf{Unifying formulation of row-normalized Muon variants.}
    \item \textbf{A study of neuron death under Muon.} We find that networks trained with Muon can have significant neuron death in the MLP layers. We present a mechanistic explanation for how neurons can die under Muon and show that this issue is effectively prevented by U-NorMuon. To verify our results, we study open-weight models trained with Muon and find strong evidence of neuron death in the dense MLP layers.
    \item \textbf{Aurora and Riemannian-Aurora.} We propose Aurora as a way to enforce uniform row norms while strictly adhering to the Muon theoretical framework. We additionally derive Riemannian-Aurora as a reference solver. Riemannian-Aurora forms updates in the tangent space of the intersection of the Stiefel manifold and the equal-row-norm manifold, more faithfully solving this dual-constraint problem at a higher cost. We find that the iterative solution used by the practical instantiation Aurora is often also very accurate in practice.
    \item \textbf{Extensive pre-training verification.} Our models trained with Aurora achieve lower downstream validation loss and higher benchmark scores than Muon and NorMuon at 340M and 1.1B scales. Notably, the 1.1B model we trained with Aurora scores 9.1 points higher on MMLU than the Muon baseline.  On the \texttt{modded-nanoGPT} speedrun (optimizers track), Aurora improves on the NorMuon baseline, and combined with existing methods it sets a new state-of-the-art at the time of writing. We additionally find that its performance gains over Muon increase with MLP width.
\end{enumerate}

\paragraph{Organization.} In Section~\ref{sec:background}, we provide background on Muon and row-normalized spectral optimizers. We additionally discuss recent research on algorithms for computing the polar factor in Muon. In Section~\ref{sec:motivation}, we show that polar factor precision and uniform row norms are concurrently desirable. We present U-NorMuon as an intermediate step to Aurora, which applies stateless normalization only to the tall matrix parameters. In Section~\ref{sec:aurora}, we formulate the dual-constraint problem, which enforces both semi-orthogonality and row-norm uniformity, and present Aurora and Riemannian-Aurora as solutions. Finally, in Section~\ref{sec:training}, we present empirical verification for Aurora and then we discuss the broader implications of our work, along with future directions in Section~\ref{sec:discussion}.

\section{Preliminaries and Related Work}\label{sec:background}

\textbf{Notation.} For a matrix $A \in \mathbb R^{m \times n}$, we denote its Frobenius norm by $\|A\|_F$,
and its spectral norm by $\|A\|_2$. The singular values of $A$ are denoted by
$\sigma_1(A) \ge \cdots \ge \sigma_{\min\{m,n\}}(A)$. The Hadamard product is
denoted by $\odot$. For a square matrix $S$, $\operatorname{diag}(S)$ denotes
the vector of diagonal entries of $S$, while for a vector $x$,
$\operatorname{diag}(x)$ denotes the diagonal matrix with diagonal $x$.
We write $\operatorname{sym}(S) = \frac{1}{2}(S+S^\top)$. For $A$ with thin SVD $A = U_r \Sigma V^\top$, we define $\operatorname{polar}(A) = U_r V^\top$. We use $d_{\mathrm{model}}$ to denote the dimension of the model's residual stream.

\subsection{Muon and its Normalized Spectral Optimizers}
Muon is an optimizer for weight matrices in the hidden layers. Its update can be viewed as the steepest first-order descent direction in a spectral-norm trust region, which is given by the polar factor of the gradient momentum. For a weight matrix $W \in \mathbb{R}^{m \times n}$ with gradient momentum $M_t$ at step $t$, Muon forms the update
\[
\Delta W_t = -\eta \operatorname{polar}(M_t) =-\eta U V^\top,\quad M_t = U \Sigma V^\top.
\]
where $\eta$ is a constant learning rate. Another perspective is that Muon performs steepest descent with respect to the spectral norm on matrix-valued parameters. Adam(W) is typically used for parameters that do not naturally represent linear maps between Euclidean spaces, such as RMSNorm scaling factors and embedding/unembedding matrices.
%Equivalently, $\operatorname{polar}(M_t)
% \in \arg\max_{\|X\|_2 \leq 1} \langle M_t, X \rangle$ is the solution to the spectral-norm
% constrained steepest descent problem $
% .

Recently, a suite of ``normalized'' Muon variants has become popular on community research benchmarks. These optimizers add an explicit row or column normalization step to the Muon algorithm either before or after the polar factor computation. For example, NorMuon maintains per-row second-moment statistics which it uses to rescale rows of the orthogonalized update matrix. NorMuon was motivated from the observation that row norms in the Muon update can be non-uniform for tall matrices. We will build on this result, showing that this non-uniformity in updates can lead to significant parameter under-utilization in the MLP layers.

MuonEq instead rescales the momentum matrix before orthogonalization, which can improve the convergence of Newton-Schulz iterations. Aurora is similar to an iterated version of MuonEq-R; however, Aurora is only applied to tall matrix parameters, whereas MuonEq-R is applied to all matrix parameters. We find that both of these differences--only normalizing updates to tall parameters, and iteratively refining the update--individually improve downstream performances.

%We find that the added benefit of MuonEq heavily depends on the precision of the NS iteration being used and that modern NS routines often converge precisely on the raw gradient momentum matrix without any additional conditioning. % TODO: need to add this.

% \begin{table}[H]
% \centering
% \renewcommand{\arraystretch}{1.4}
% \begin{tabular}{@{}ll@{}}
% \toprule
% Optimizer & Update direction \\
% \midrule
% Muon    & $\operatorname{polar}(M_t)$ \\
% MuonEq  & $\operatorname{polar}\!\big(\operatorname{diag}(\rho_t)^{-1} M_t\big)$ \\
% NorMuon & $\operatorname{diag}\!\big(\sqrt{v_t}+\epsilon\big)^{-1}\operatorname{polar}(M_t)$ \\
% \bottomrule
% \end{tabular}
% \caption{Update directions for Muon and two row-normalized variants, for a tall momentum buffer $M_t \in \mathbb{R}^{m \times n}$ ($m > n$).}
% \label{tab:muon-variants}
% \end{table}

%TODO: I don't think i'm happy with this para and I'm not sure that it's necessary at all
A related but distinct line of work uses normalization to improve the memory and compute efficiency of optimizers, applying stateless row- or column-wise rescaling to gradients as a lightweight surrogate for adaptive preconditioning, or for orthogonalization itself~\citep{ma2024swan, scetbon2025gradientmultinormalization, glentis2025minimalist, wen2025sron,xu2026widthscaling,deng2026rmnp}.  Aurora's row-normalization step is mechanically similar to several of these methods but has a different motivation. Aurora row-normalizes \textit{and} orthogonalizes the update, achieving a precise row-uniform polar factor. As we show in Section~\ref{sec:normalization}, there is a natural tension between row normalization and orthogonalization for tall matrices; resolving that tension is the central design problem that Aurora addresses.

%The motivation for Aurora is most closely tied to that of NorMuon though we arrive at a mechanism that is in part similar to some of the other optimizers described above. 

%Although Aurora is closest in motivation to NorMuon, its mechanism connects to this broader normalization literature by treating row-wise update scale as a first-class optimization variable.

%methods replace expensive optimizer state or matrix preconditioning with stateless row/column normalization, whitening, or multi-norm gradient scaling. 

% A separate line of work motivates column/row normalization from the perspective of improving memory and compute efficiency. SWAN statelessly preprocesses gradients with normalization and whitening, and Gradient Multi-Normalization frames stateless optimization as normalization with respect to multiple norms. SCALE shows that column-wise normalization can improve the performance of plain SGD, and SRON studies the complementary row-wise variant, motivated by large disparities in gradient scale across rows of weight matrices.

% RMNP replaces NS iteration with row-wise momentum normalization, which can be interpreted as a structured approximation to the Muon preconditioner. Mano proposes a manifold-inspired optimizer that projects momentum relative to normalized parameters and constrains the update direction with a rotating Oblique-manifold normalization, equivalent to alternating row-wise and column-wise normalization across steps.

\subsection{Algorithms for Computing the Polar Factor}

%Practical Muon implementations use Newton--Schulz-type iterations to approximate the polar factor using only matrix multiplications, often by implicitly applying an inverse-square-root iteration to the Gram matrix of the momentum.

A \textit{Newton-Schulz (NS) iteration} is a fixed-point recurrence of the form $X_{t + 1} = p(X_t)$ for some matrix polynomial $p$. Practical Muon implementations usually use an NS-type iteration to approximate the polar factor, often by implicitly applying an inverse-square-root iteration to the Gram matrix. For a matrix $G \in \mathbb{R}^{m \times n}$, these methods produce an approximation $\widehat P_T(G) \approx \operatorname{polar}(G)$ after $T$ iterations. The existence of matmul-only algorithms for computing the polar factor is largely what makes Muon feasible at scale.

Prior work finds that more precise NS iterations (i.e. iterations producing $\widehat P_T(G)$ that is closer to $\operatorname{polar}(G)$) lead to better downstream performance but that gains diminish quite quickly beyond a point of saturation~\citep{kim2026convergencemuonnewtonschulz}. We will consider three commonly used algorithms in this paper: Jordan's quintic polynomial with five steps (quintic-5)~\citep{jordan2024muon}, Polar Express with eight steps (PE-8)~\citep{amsel2025polarexpress} and CANS with degree-3 and twelve steps (CANS-12)~\cite{grishina2025accelerating}. CANS-12 and PE-8 both require 24 matrix multiplications, while quintic-5 requires only 15. We find that this difference in computation results in a negligible  wall-clock overhead in our distributed Muon setting.% We choose these step counts following community standards.% and our own ablations which show performance saturation at or near these points (Appendix~X)

%Quintic-5 is the cheapest of these algorithms but is less precise than CANS-12 and PE-8 (Figure~\ref{fig:ns_polys}). We find that using PE-8 or CANS-12 results in only a very modest wall-clock overhead in our distributed Muon implementation.  %TODO: I need an appendix ablation if I want to say this. something better here probably

%In practice, it can be difficult to compare different NS iterations with small iteration budgets because convergence often depends on the shape and conditioning of the input matrix. For the purposes of this study, we evaluate NS routines using the validation loss they achieve when used with Muon in our training setting. Our setup and choice of model architecture are all very standard, so we expect results to transfer cleanly to other settings.

\begin{figure}[H]
    \centering
    \includegraphics[width=0.7\linewidth]{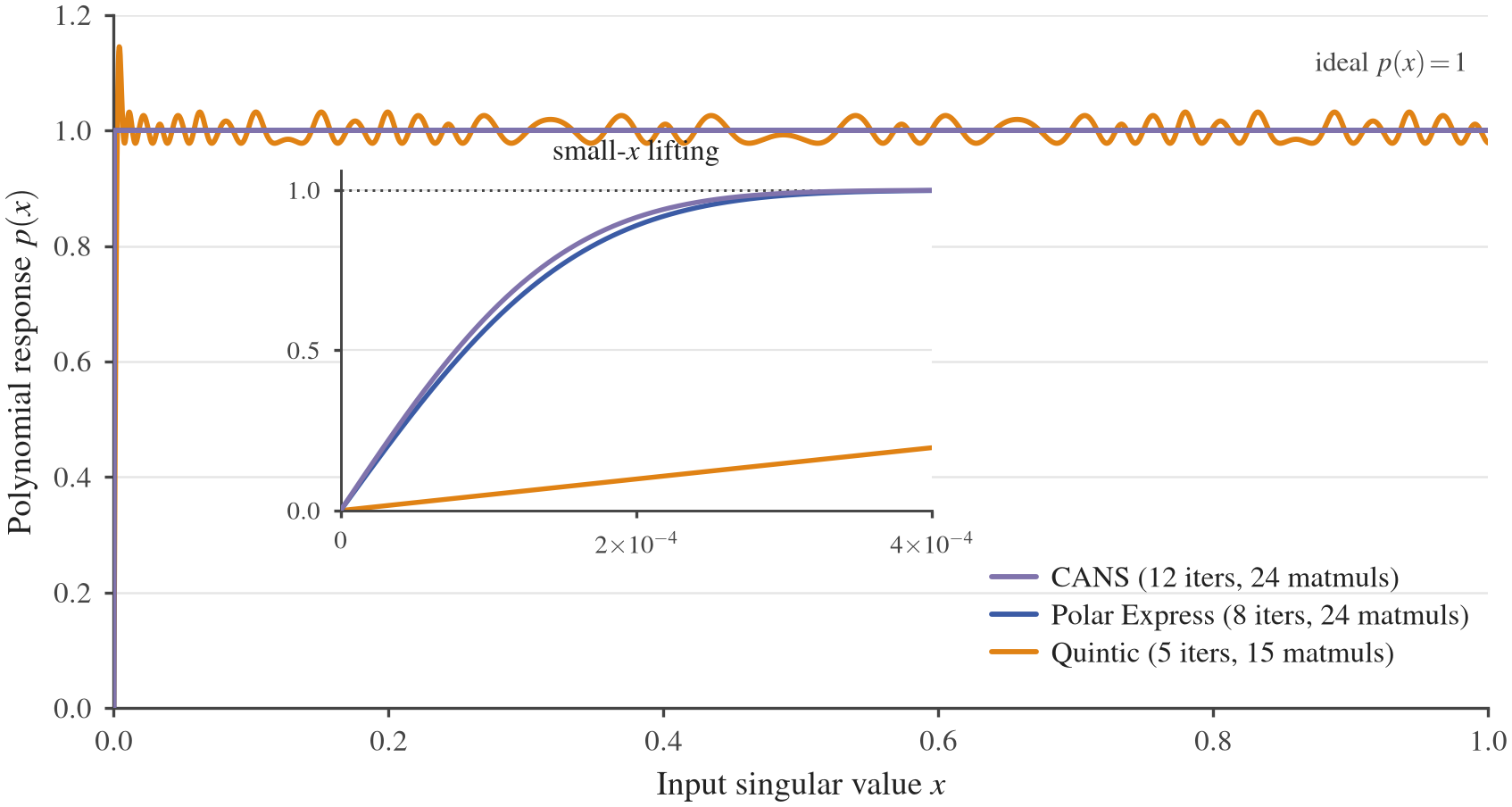}
    \caption{Plotting CANS-12, quintic-5 and PE-8 for singular values between zero and one. Quintic-5 has the largest error, and CANS-12 converges slightly better than PE-8 for small singular values.}
    \label{fig:ns_polys}
\end{figure}

%It is hard to make an absolute comparison of different NS iterations because convergence largely depends on the input's shape and conditioning. For the purposes of this study, we will compare these methods using downstream validation loss and empirical approximation error.

Intuitively, more precise methods map singular values closer to one on average. CANS-12 and PE-8 are both very accurate for large singular values, but for small singular values, CANS-12 can be noticeably more precise (Figure~\ref{fig:ns_polys}). In Section~\ref{sec:normalization}, we will empirically measure the extent of this precision gap in our training settings.

\section{Polar Precision and Neuron Update Imbalance}\label{sec:motivation}
We study different NS iterations in our  settings and find monotonically lower validation loss for higher polar factor precisions. Motivated by this result, we study the effect of normalization on the polar factor, and show that current row/column normalization schemes necessarily reduce polar precision for tall and wide matrix parameters respectively (Section~\ref{sec:normalization}).

% maybe we need graph of ||polar(G) - NS(G)||
We then study parameter efficiency in networks trained with Muon and find evidence of significant neuron death in the MLP layers. We show that row normalization of the Muon update effectively mitigates this pathology (Section~\ref{sec:death}), leading to improved downstream evaluation results.

We verify the generality of our results by inspecting open-weight models for dead neurons (Section~\ref{sec:open_weight}). In all cases, we find a non-trivial cohort of dead neurons particularly concentrated in the early layers. Consistent with our proposed mechanistic explanation, we do not find dead neurons in fine-grained experts in the models we inspect.

%that a small but highly consistent set of neurons in an open-weight Muon-trained model are dead by our criteria, concentrated in the early layers and stable across evaluation datasets.

%Before studying normalized variants, we first verify that polar-factor precision
%is relevant in our training setting. Holding all other details fixed, we vary only
%the orthogonalization routine used by Muon. Higher-precision routines achieve
%lower validation loss, suggesting that perturbations away from the polar factor
%can have measurable downstream consequences. This motivates our goal of
%correcting row update imbalance without sacrificing polar precision.

\subsection{Analyzing Normalization Schemes}\label{sec:normalization}
It is hard to make an absolute comparison of different NS iterations because their convergence largely depends on the input's shape and conditioning. For the purposes of this study, we will compare the accuracy of these methods using the following empirical error measure.

\begin{definition}[Polar approximation error (PAE)]
For a matrix \(X \in \mathbb{R}^{m \times n}\) and a Newton--Schulz iteration
\(\operatorname{A}\), let \(\operatorname{A}_k(X)\) denote the result of applying \(\operatorname{A}\) to \(X\) for \(k\)
steps. We define the \textit{polar approximation error} of \(\operatorname{A}\) on \(X\) after
\(k\) iterations as
\[
\varepsilon_{k,\operatorname{A}}(X)
\coloneqq
\frac{
\left\|\operatorname{A}_k(X)-\operatorname{polar}(X)\right\|_{\mathrm F}
}{
\left\|\operatorname{polar}(X)\right\|_{\mathrm F}
}.
\]
\end{definition}
We train a vanilla 340M transformer with Muon using different NS iterations and compute the average polar approximation error over all matrix parameters updated by Muon for each step of training (full training details given in Section~\ref{sec:training}). Intuitively, this metric captures the average error introduced by approximate orthogonalization. We additionally run Muon with singular value decomposition (SVD) as a perfect-accuracy baseline.

\begin{table}[H]
\centering
\begin{tabular}{lrrrrr}
\toprule
Algorithm & \(k\) & Val Loss  & Mean PAE & Max PAE & PAE Std \\
\midrule
Quintic       & 5  & 2.607 & 0.1556 & 0.9244 & 0.1809 \\
Polar Express & 8  & 2.584 & 0.0298 & 0.7136 & 0.0618 \\
CANS          & 12 & 2.583 & 0.0286 & 0.7064 & 0.0601 \\
SVD & -- & 2.577 & 0.0017$^{\dagger}$ & 0.0017 & -- \\
% muon_svd_340M                               8  polar_express       0.001660      0.001668      0.000002      96    480096    
\bottomrule
\end{tabular}
\caption{Polar approximation error statistics and validation loss computed across 340M training trajectories with different NS iterations and SVD. $^\dagger$SVD achieves a non-zero PAE because training is done in BF16.}
\label{tab:polar-approx-error}
\end{table}

We find a clean correspondence between lower PAE and lower validation loss. CANS-12 achieves both the best validation loss and mean PAE of all the polynomials we test, but is outperformed by SVD. PE-8 performs similarly to CANS-12, and both of these algorithms do significantly better than quintic-5, which is much less precise. %These results provide strong verification that higher polar precision is empirically desirable.

%where $G_{A,p}^{(i)}$ is the raw gradient for parameter $p$ at training step $i$ and $\Theta$ is the set of all parameters that are updated by the spectral optimizer. This metric captures the average precision error of polar factors computed in a particular training run. We use a vanilla transformer and standard training presets in an effort to maximize the transferability of our results. %We emphasize that the gradients under $A_k$ and $B_q$ will almost always be different and that $G_i$ depends on $G_j$ for all $j < i$. 

\begin{figure}[H]
    \centering
    \includegraphics[width=0.8\linewidth]{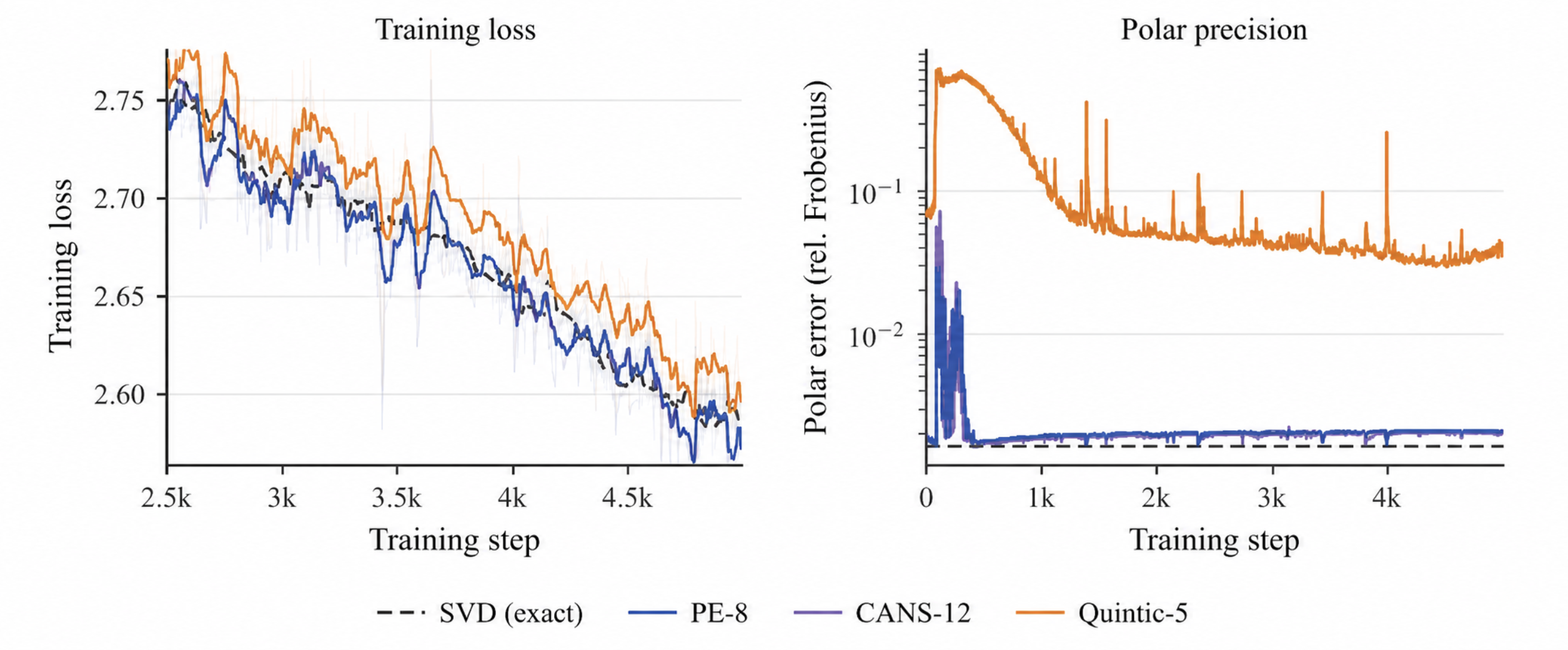}
    \caption{Loss curves and polar approximation error graphs for different NS iterations in our 340M Muon setting. Quintic-5 gets a much larger polar error and worse training loss than both PE-8 and CANS-12. SVD achieves the best and most stable training loss.}
    \label{fig:340M_poly}
\end{figure}

We analyze the effect of different row-normalization schemes and find that methods enforcing unit row norms necessarily change the geometry of the polar factor for tall matrices. NorMuon smooths its row-normalization weights with an exponential moving average and therefore does not always produce exactly unit row norms, but we still treat it as belonging to this class. Symmetrically, column-normalized variants of these optimizers induce a larger approximation error for \textit{wide} matrix updates. 

\begin{prop}\label{claim:tall}
Let $A \in \mathbb{R}^{m_A \times n_A}$ be a tall matrix parameter such that $m_A > n_A$. The matrix $A$ cannot be both column-orthogonal and have unit row norms. Similarly, the wide matrix $B \in \mathbb{R}^{m_B \times n_B}$ with $m_B < n_B$ cannot be both row-orthogonal and have unit column norms. Lastly, if $C \in \mathbb{R}^{m_C \times n_C}$ is orthogonal and $m_C = n_C$, then its rows and columns must both have unit norm.
\end{prop}
\indent All proofs are given in Appendix~\ref{sec:proofs}.

Thus, there is a ``competition'' between post-hoc normalization and the polar factor computation in unit-normalized spectral optimizers. When the momentum buffer has large variance in row norms, the precision reduction due to unit row-normalization can be significant. In particular, the loss of precision can be greater than the precision gap between quintic-5 and CANS-12.%, which we found leads to poorer performance for models trained with Muon (Table~\ref{tab:polar-approx-error}).

We study this in a simple synthetic experiment in which we sample tall matrices with different row-norm spreads and examine the effect of unit row-normalization on polar factor precision. As the row-norm spread $\sigma$ grows (Figure~\ref{fig:unit_row}), the input becomes increasingly heavy-tailed across rows. This degrades the precision of the NS-computed polar factor. The additional orthogonality defect introduced by unit row normalization is largest at intermediate $\sigma$, where the raw NS update remains relatively precise but row normalization substantially changes its geometry.

\begin{figure}[H]
    \centering
    \includegraphics[width=0.9\linewidth]{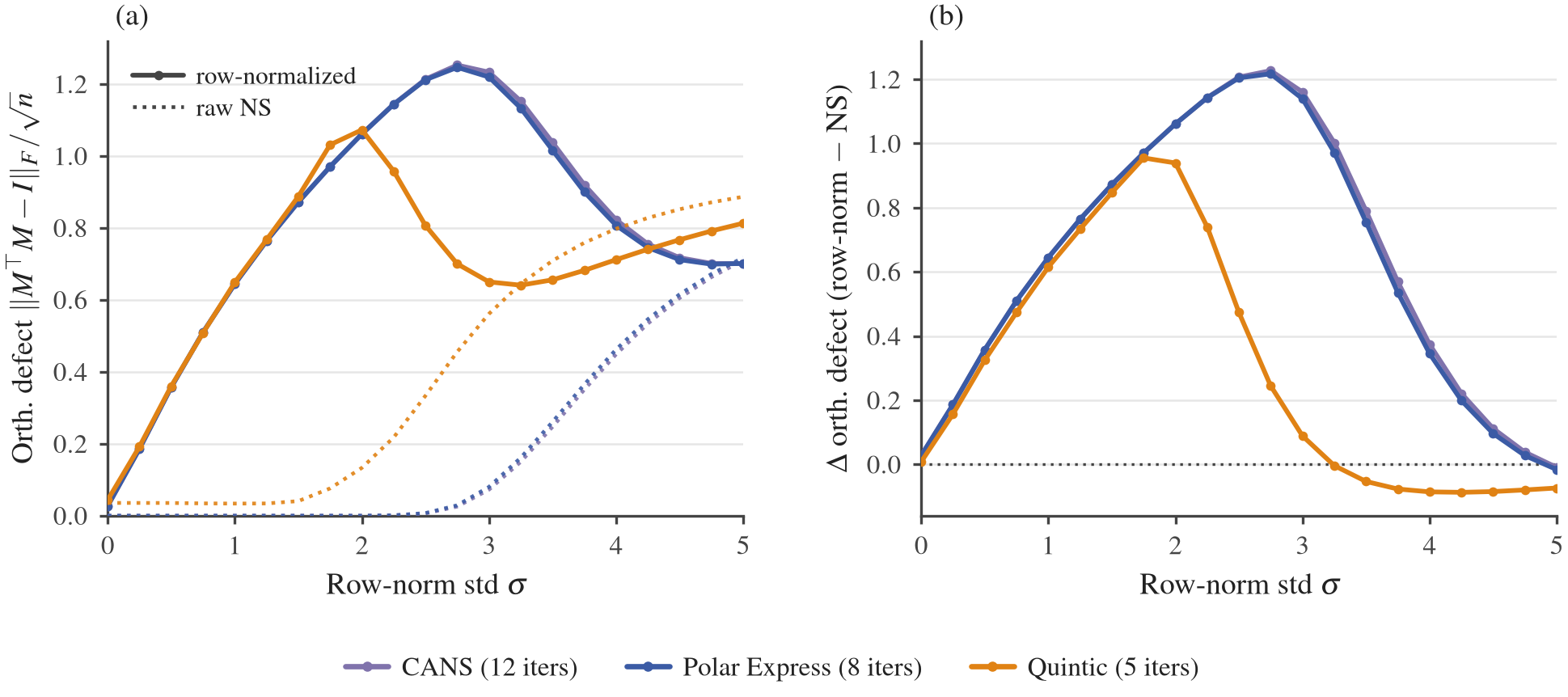}
    \caption{
Effect of unit row normalization on orthogonality for random $512\times128$ matrices with log-normal row norms. Rows are sampled as random unit directions and scaled by $r_i=\exp(\sigma z_i)$, where $z_i\sim\mathcal{N}(0,1)$. \textbf{(a)} Orthogonality defect before and after unit row normalization. \textbf{(b)} Additional defect introduced by row normalization.
}
    \label{fig:unit_row}
\end{figure}

Proposition~\ref{claim:tall} also implies that normalization of wide/square matrices becomes redundant under more precise NS iterations. For optimizers that maintain state for normalization variables (e.g. momentum buffers for row norms), normalization may even be harmful if the state becomes stale at some point in training. We find that NS iterations tend to be less precise early in training when raw gradients are more likely to be ill-conditioned (Figure~\ref{fig:340M_poly}), allowing for the accumulation of normalization statistics that quickly become stale. In this case, normalization may make rows/columns \textit{less} uniform.

\begin{prop}\label{claim:stateless}
    Let $M_t \in \mathbb{R}^{m \times n}$ be the momentum buffer at training step $t$ and let $O_t = \operatorname{polar}(M_t)$ be the precise Muon update. Let $r_{t,i}
=
\operatorname{RMS}\!\left((O_t)_{i,:}\right)$ be the RMS norm of the $i^{th}$ row of $O_t$ and assume $r_{t, i} > 0 \quad \forall i \in [m]$. If $\widehat O_t$ is the \textit{stateless} row-normalized update defined as
\[
\widehat O_{t,i,:}
=
\frac{(O_t)_{i,:}}{r_{t,i}},
\]
then $\operatorname{RMS}(\widehat O_t)=1$.
\end{prop}

By Proposition~\ref{claim:stateless}, the updates produced by stateless row-normalization naturally have unit RMS norm, matching the update RMS of Muon with spectral scaling~\cite{liu2025muon}. However, the row normalization operation can significantly alter the polar factor geometry and may amplify singular directions that would normally be given small weight under Muon. We experiment with \textit{polar scaling}, which maintains the RMS norm of the polar factor after row rescaling; for NorMuon, this amounts to rescaling the update by a factor of $1/\sqrt{m}$. We find polar scaling to be more effective in our small-scale experiments.
%Additionally, normalization can modify parameter update scale significantly. For example, the true Muon update for a tall matrix parameter $M \in \mathbb{R}^{m \times}$ satisfies,
% $$\|O\|_F^2 = n, \quad O = \operatorname{polar}(M)$$
% NorMuon then scales each row of $O$ to have approximately unit RMS norm, producing the update $\hat{O}$. The scale of this update will be larger by a factor of roughly $m$.
% \[
% \|\widehat O_{i,:}\|_2^2 \approx n
% \qquad \text{for each } i \in [m],
% \]
% and therefore
% \[
% \|\widehat O\|_F^2
% =
% \sum_{i=1}^m \|\widehat O_{i,:}\|_2^2
% \approx
% mn.
% \]
% We hypothesize that this scale mismatch may negatively impact the learning dynamics, especially for very tall matrices which will get much larger updates under NorMuon.

Combining these observations yields U-NorMuon, which applies row normalization to the momentum polar factor for updates to tall matrices, and otherwise uses Muon for the wide/square hidden weight matrices. We leave the choice of scaling for updates to tall parameters as a hyperparameter; unless otherwise specified, we use polar scaling which performed the best in our setting; see Appendix~\ref{app:u_normuon} for a full ablation. Parameters optimized with Muon use spectral scaling in all cases. Because it is stateless, U-NorMuon uses $O(m)$ less memory than NorMuon for each tall matrix.

\begin{algorithm}[H]
\caption{U-NorMuon}
\label{alg:u-normuon}
\begin{algorithmic}[1]
\Require Initial weights $\mathbf W_0 \in \mathbb R^{m \times n}$, loss $L$, learning rate $\eta$, momentum coefficients $\beta$, weight decay $\lambda$, numerical constant $\epsilon$, tall update scaling $s : \mathbb{R}^{2} \rightarrow \mathbb{R}$
\State Initialize $\mathbf M_0 \in \mathbb R^{m \times n} \gets \mathbf 0$
\For{$t = 1,2,\ldots$}
    \State $\mathbf G_t \gets \nabla_{\mathbf W} L(\mathbf W_t)$
    \State $\mathbf M_t \gets \beta_1 \mathbf M_{t-1} + (1-\beta_1)\mathbf G_t$
    \State $\mathbf O_t \gets \operatorname{polar}(\mathbf M_t)$
    \If{$m > n$} \Comment{Only row-normalize tall matrices}
        %\State $\mathbf v_t \gets \beta_2 \mathbf v_{t-1}
        %     + (1-\beta_2)\operatorname{mean}_{\mathrm{cols}}(\mathbf O_t \odot \mathbf O_t)$
        % \State $\mathbf V_t \gets \operatorname{ExpandRows}(\mathbf v_t)$ \Comment{$\mathbf V_t \in \mathbb R^{m \times n}$}
        \State $\mathbf r_t \gets \operatorname{mean}_{\mathrm{cols}}(\mathbf O_t \odot \mathbf O_t)$
        \State $\mathbf R_t \gets \operatorname{ExpandRows}(\mathbf r_t)$
        \State $\widehat{\mathbf O}_t \gets \mathbf O_t \oslash \left(\sqrt{\mathbf R_t} + \epsilon\right)$
        %\State $\gamma \gets 1/\sqrt{m}$ \label{step:polar_scale}\Comment{Polar scaling}
        \State $\gamma \gets s(n, m)$ \Comment{Polar or spectral scaling}
    \Else
        \State $\widehat{\mathbf O}_t \gets \mathbf O_t$
        \State $\gamma \gets \sqrt{n}$ \Comment{Spectral scaling}
    \EndIf
    \State $\mathbf W_{t+1} \gets \mathbf W_t - \eta \lambda \mathbf W_t - \eta \gamma \cdot \widehat{\mathbf O}_t$
\EndFor
\end{algorithmic}
\end{algorithm}

U-NorMuon outperforms both NorMuon and Muon in our 340M transformer setting where all optimizers use CANS-12 for NS. Additionally, both row-normalized optimizers outperform our Muon baseline, providing isolated evidence that tall matrices benefit from row-uniform updates.
%We also individually ablate all of our design choices and find U-NorMuon performs the best (Appendix~\ref{app:u_normuon}). 

 % sIn Section~\ref{sec:aurora}, we build on U-NorMuon, showing that we can achieve uniform row norms while preserving the polar factor geometry.

\begin{figure}[H]
    \centering
    \includegraphics[width=0.6\linewidth]{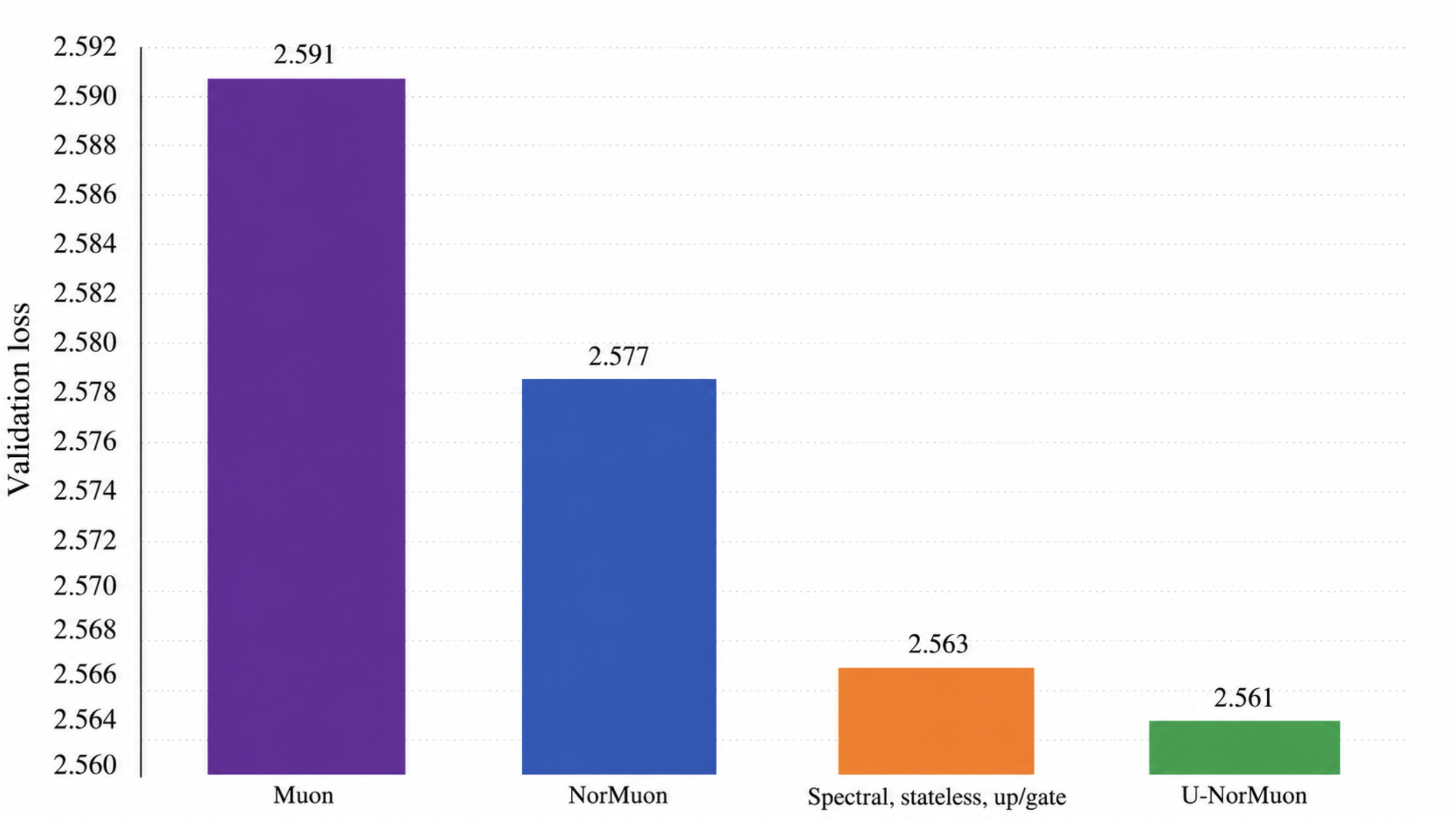}
    \caption{U-NorMuon achieves better downstream loss than both our Muon and NorMuon baselines at 340M scale. We additionally plot the spectral scaling ablation, which achieves slightly higher validation loss than U-NorMuon which uses polar scaling on tall matrix updates by default.}
    \label{fig:u_normuon}
\end{figure}

Motivated by these results, we seek to understand why updates to tall parameters can benefit from row normalization. We build on prior work by showing mechanistically how neurons can receive persistently small learning signal under Muon and we find extensive evidence that this occurs in our settings and in open-weight models, sometimes resulting in neuron death.

\subsection{Neuron Death Under Muon}\label{sec:death}
Motivated by activation-based definitions of dead and dormant neurons
\citep{lu2020dyingrelu, sokar2023dormant, dohare2024lossofplasticity},
and by gradient-based saliency criteria for pruning
\citep{lee2019snip, wang2020grasp, molchanov2019importance},
we define an \textit{optimizer-relative dead component} as one that persistently receives
both low effective gradient signal and low realized update magnitude. We stress that some model components may get small updates/raw gradients at different points in training but still contribute meaningfully to model outputs. We distinguish these components from genuinely dead ones by their gradient/update norm patterns \textit{across} training, which we find tend to be meaningfully different for dead and alive components.

% \begin{definition}[Optimizer-relative dead component]
% Let \(\theta = \{\theta^{(c)}\}_{c\in\mathcal C}\) be a decomposition of a parameter tensor into components, such as the rows of an MLP projection matrix. At training step \(t\), let
% \[
% g_t^{(c)}=\nabla_{\theta^{(c)}}\mathcal L_t,
% \qquad
% u_t^{(c)}=\theta_{t+1}^{(c)}-\theta_t^{(c)}
% \]
% denote the raw gradient and realized optimizer update for component \(c\). Define the step-normalized gradient and update magnitudes
% \[
% \bar g_t^{(c)}
% =
% \frac{\|g_t^{(c)}\|_2}
% {\operatorname{median}_{c'\in\mathcal C}\|g_t^{(c')}\|_2+\epsilon},
% \qquad
% \bar u_t^{(c)}
% =
% \frac{\|u_t^{(c)}\|_2}
% {\operatorname{median}_{c'\in\mathcal C}\|u_t^{(c')}\|_2+\epsilon}.
% \]
% For a training window \(I\), we say that component \(c\) is
% \((\tau_g,\tau_u,\rho)\)-optimizer-dead on \(I\) if
% \[
% \frac{1}{|I|}
% \sum_{t\in I}
% \mathbf 1
% \left\{
% \bar g_t^{(c)}\leq \tau_g
% \;\;\text{and}\;\;
% \bar u_t^{(c)}\leq \tau_u
% \right\}
% \geq \rho .
% \]
% \end{definition}

A defining feature of MLP layers is that their neurons are computed independently along rows. Because the MLP nonlinearity acts coordinate-wise, each hidden unit $h_i$ is a function of only the
$i^{th}$ rows of the input projections, so the latent space decomposes into per-neuron computations which do not share parameters. This row structure is precisely what makes MLPs vulnerable to gradient suppression and neuron death, since when the update to a row is small relative to the others, the corresponding neuron learns slowly or stops contributing altogether independently of the rest of the layer. Since Muon orthogonalizes the gradient matrix as a whole, understanding which neurons are amplified or suppressed reduces to understanding how Muon transforms the row norms of the raw gradient/momentum buffer. 

Concretely, we focus on the up and gate projections which are both tall matrix parameters. In a SwiGLU MLP~\cite{shazeer2020gluvariants} the $i^{th}$ coordinate of the latent for $x \in \mathbb{R}^{d_{model}}$ given by
$$h_i = \operatorname{SiLU}((W_{gate})_{i,:}x) \cdot (W_{up})_{i,:} x$$
In particular, the neuron $h_i$ only depends on $W_{gate}$ and $W_{up}$ through their $i^{th}$ rows. The same holds for an arbitrary coordinate-wise nonlinearity $\phi : \mathbb{R}^r \rightarrow \mathbb{R}$; we can write the $i^{th}$ hidden coordinate as
\[
    h_i = \phi(z_{i,1},\ldots,z_{i,r}),
    \qquad
    z_{i,j} = (W^{(j)})_{i,:}x,
\]
where \(W^{(1)},\ldots,W^{(r)}\) are the input projections. For example, in a SwiGLU MLP we have $r=2$ with $W^{(1)} = W_{up}$, $W^{(2)} = W_{gate}$; in a ReLU$^p$ MLP for some $p > 0$, we have $r = 1$ and $W^{(1)} = W_{up}$. The coordinate $h_i$ depends on these matrices only through their $i^{th}$ rows. Our empirical results in this section are all for models using SwiGLU MLPs but we show the results are broadly the same for the $\operatorname{ReLU}^2$ activation in Appendix~\ref{app:relu}.

We will use \textit{leverage scores} to formalize the connection between row norms in the raw gradient matrix and the corresponding neuron-level updates under Muon. %Leverage scores quantify the relative strength of a row's participation in the update subspace. 

\begin{definition}[Leverage scores]
For a matrix $M \in \mathbb R^{m \times n}$ with thin SVD
$ M = U_r  \Sigma  V^\top$, define the leverage
score of row $i$ as
\[
\ell_i(M) = \|(U_r)_{i,\cdot}\|_2^2.
\]
\end{definition}

Unlike gradient row norms, row leverage scores are not influenced by singular-value scale, which is made approximately uniform under Muon. They quantify the relative strength with which a row participates in the column space of the update matrix. Second, for tall matrices, the leverage score of a row in the momentum is exactly the squared norm of the corresponding row in the polar factor.

%Because idealized Muon updates a parameter by $\operatorname{polar}(M_{t, W})$, a row's leverage score in the momentum buffer \emph{is} the magnitude of the Muon update to the corresponding neuron. The following bound then shows that rows with small gradient signal are assigned small leverage, and hence small updates.

\Needspace{8\baselineskip}
\begin{prop}\label{claim:leverage}
    Let $M \in \mathbb{R}^{m \times n}$ where $m > n$ with thin SVD $M = U_r\Sigma V^T$. Suppose that for some $\epsilon > 0$ row $i$ satisfies $\|M_{i, \cdot}\| \leq \epsilon$, then the row's leverage score satisfies 
    $$\ell_i(M) = \|\operatorname{polar}(M)_{i, :}\|^2 \leq \epsilon^2/\sigma_n(M)^2,$$ 
    where $\sigma_n(M)$ is the smallest singular value of $M$.
\end{prop}

%Let $M_t$ be the gradient first momentum buffer for a tall matrix parameter (i.e. up or gate projection) at step $t$ and suppose $M_t$ has thin SVD $M_t = U_r \Sigma V^T$. By Claim~\ref{claim:leverage}, if $\|(M_t)_{i, \cdot}\| \leq \epsilon$ then $\ell_i(M_t) \propto \epsilon^2$. So in particular, if row $i$ has small norm in the raw momentum then $\epsilon$ is small and row $i$ also has small leverage. Muon will form the update $\operatorname{polar}(M_t) = U_rV^{\top}$, where $\|(U_r)_{i, \cdot}V^{\top}\|_2^2 = \|(U_r)_{i, :}\|^2_2 = \ell_i(M_t)$.

If $M_{t, W} \in \mathbb{R}^{m \times n}$ is the gradient momentum for a parameter $W$ at step $t$, then an idealized version of Muon will update $W$ according to $W \gets \operatorname{polar}(M_{t, W})$. Suppose that $W$ is a tall parameter and that $W_{i, :}$ has received persistently small gradients in recent steps, implying that $\|(M_{t, W})_{i, :}\|_2^2$ is small. Then by Proposition~\ref{claim:leverage}, when $M_{t, W}$ is reasonably well-conditioned, row $i$ will also have small magnitude in the Muon update. Thus, neuron update magnitude non-uniformity is preserved by Muon when $\sigma_n(M_{t,W})$ is not too small.
%Suppose that $M_t$ is the gradient momentum buffer for a tall matrix parameter (i.e. up or gate projection) at step $t$ and suppose $M_t$ has thin SVD $M_t = U_r \Sigma V^T$. By Claim~\ref{claim:leverage}, if $\|(M_t)_{i, \cdot}\| \leq \epsilon$ then $\ell_i(M_t) = \|\\propto \epsilon^2$. 

This predicts that the Muon momentum buffer can accumulate persistently small neuron-level updates, which we find occurs in our settings. We find a strong correspondence between neurons receiving small updates and neurons with low row leverage, indicating that the bound in Proposition~\ref{claim:leverage} is tight (Figure~\ref{fig:dead_correlation}). U-NorMuon eliminates these dead neurons in all of our settings and achieves lower downstream validation loss.

\begin{figure}[H]
    \centering
    \includegraphics[width=0.7\linewidth]{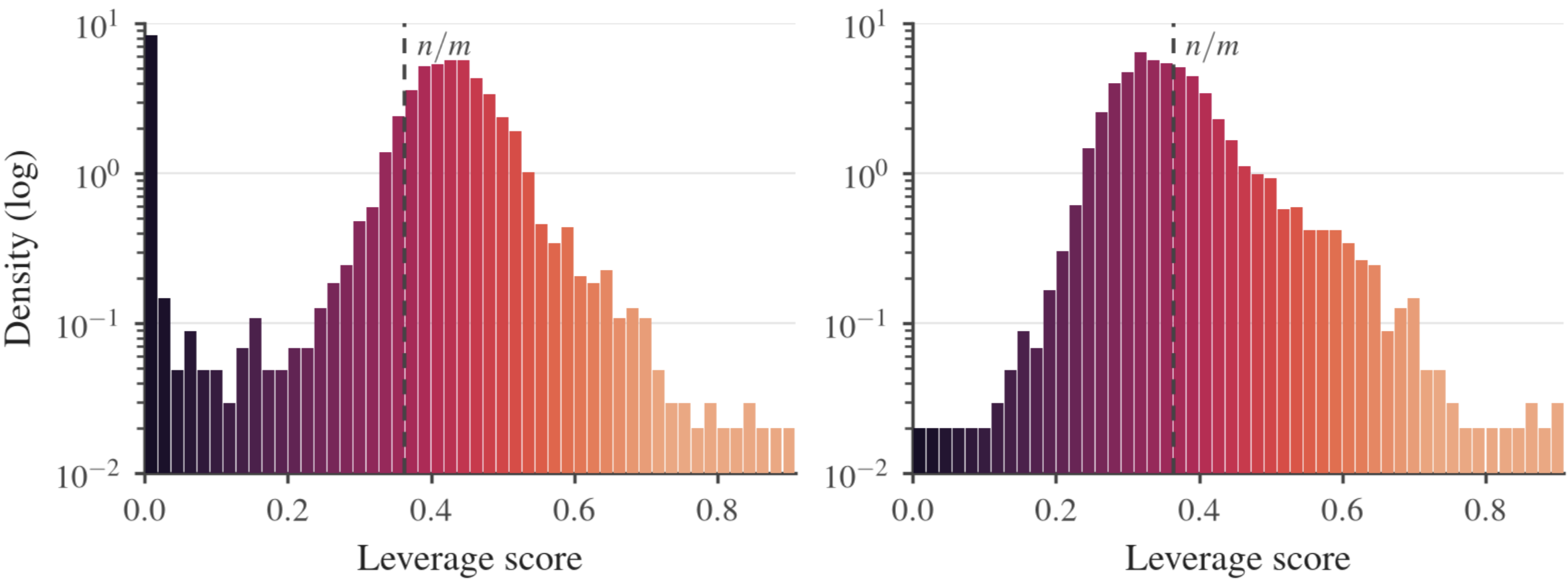}
    \caption{Each square on the grid corresponds to a row in the parameter matrix; darker neurons have smaller row leverage.  The distribution of leverage scores for Muon is bimodal, whereas U-NorMuon leverage scores are isotropic and there are no dead neurons. Leverage score statistics are computed from the momentum buffer at the final step of training.}
    \label{fig:dead_neurons}
\end{figure}

We plot the distribution of leverage scores at different points in training and find that after some initial training instability, distributions for both U-NorMuon and Muon remain mostly static. Our models trained with Muon form a bimodal leverage score distribution with one mode near the uniform value of $n/m$ and the other near zero; this latter mode corresponds to dead neurons. Conversely we find that our models trained with U-NorMuon have leverage scores tightly clustered around $n/m$ (Figure~\ref{fig:dead_neurons}).

\begin{figure}[H]
    \centering
    \includegraphics[width=0.6\linewidth]{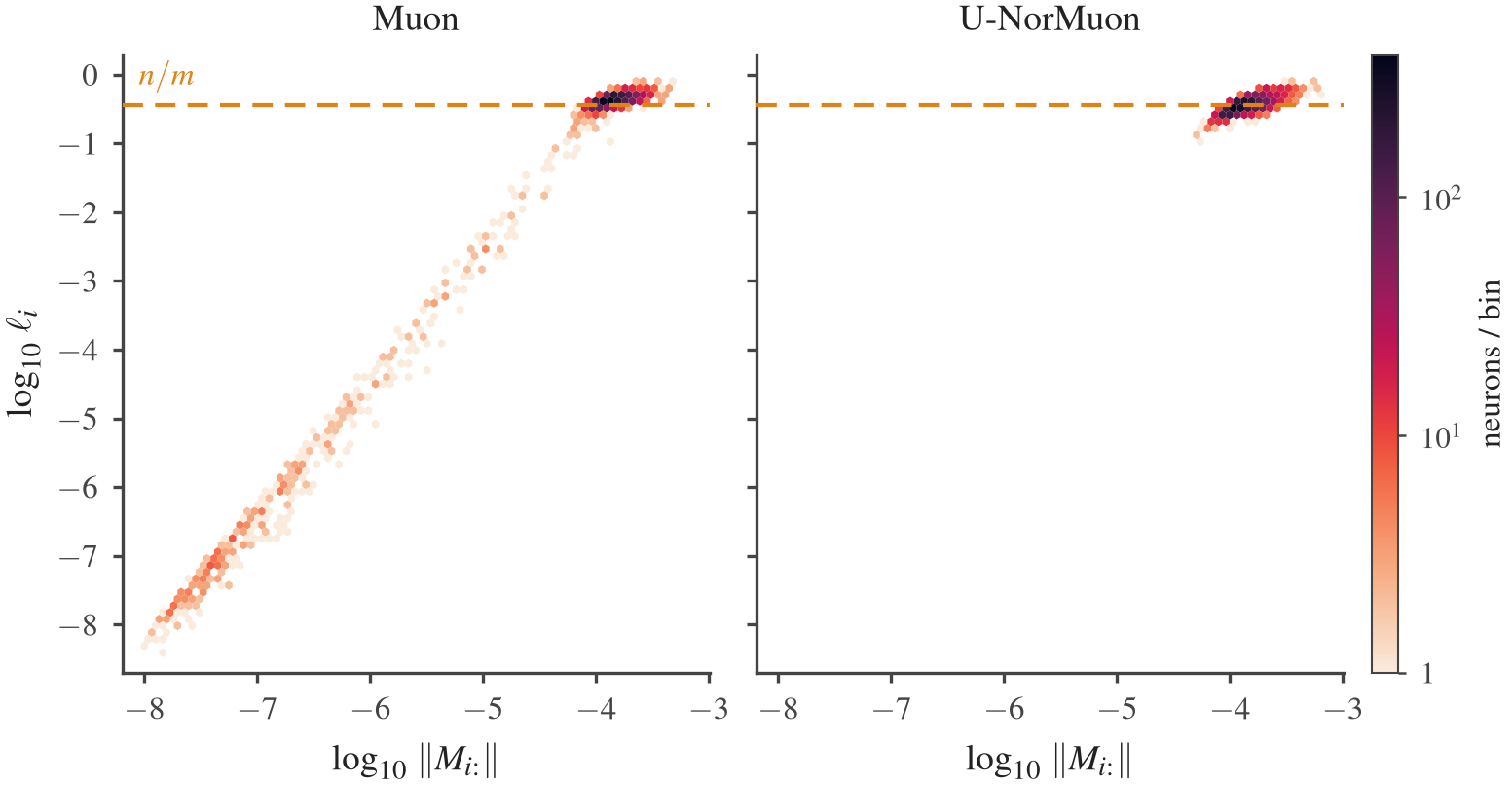}
    \caption{We plot the correlation between the leverage scores and row norms of the momentum buffer at the final step of training. Muon has a long tail of neurons with low update magnitude and leverage, but U-NorMuon does not. The momentum row norms and leverage scores are strongly correlated on a log-log scale.}
    \label{fig:dead_correlation}
\end{figure}

Notably, we find that the row leverage scores and the row norms of the momentum buffer both tend towards the uniform value under U-NorMuon, indicating that the leverage score correction applied by U-NorMuon becomes increasingly small at later steps. In other words, the gradient momentum buffer converges to a stable regime with uniform row norms and leverage as training progresses (Figure~\ref{fig:converge}). 
%U-NorMuon creates a self-reinforcing geometry that fundamentally alters training dynamics.

\begin{figure}[H]
    \centering
    \includegraphics[width=0.8\linewidth]{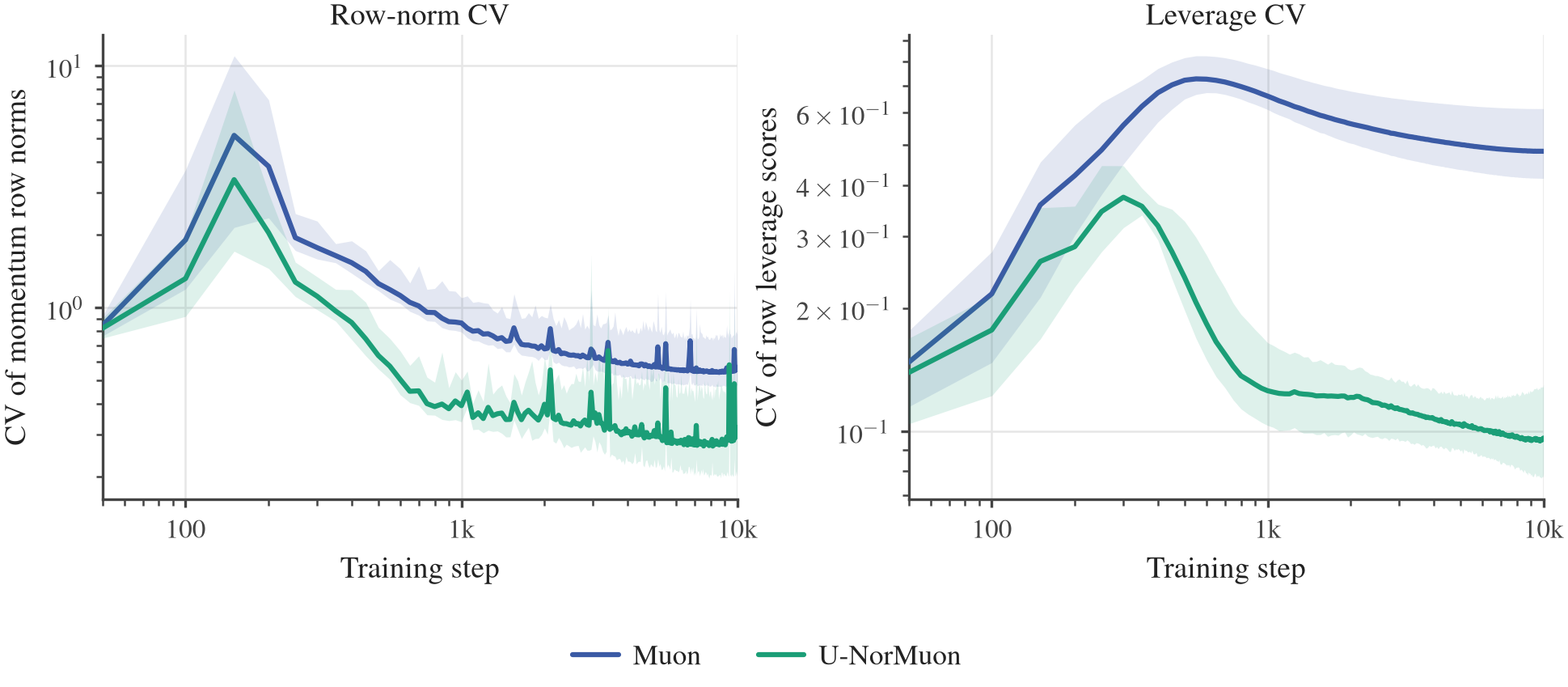}
    \caption{Coefficient of variation (CV) of momentum row norms (left) and row leverage scores (right) for Muon and U-NorMuon (mean $\pm$ standard deviation across layers). U-NorMuon converges to far more uniform row norms and leverage, while Muon's momentum buffer remains highly anisotropic.}

    \label{fig:converge}
\end{figure}

We additionally find that column leverage scores in the down projection matrices tend to be uniform under U-NorMuon but anisotropic under Muon. Down projection matrices are wide and hence receive the same update under both Muon and U-NorMuon. We hypothesize that low-leverage rows in the up/gate projections lead to the corresponding columns in the down projection matrix receiving small gradient signal and hence small updates. %In other words, dead neurons are simultaneously a pathology of the up/gate projection rows and of the down projection columns. 

\subsection{Neuron Death in Open Weight Models}\label{sec:open_weight}
We probe for neuron death in open-weight models to verify that our results transfer to other training settings. We are not aware of any open-weight models trained with Muon that are also open-data, so we evaluate models on a suite of representative open-source datasets. In Section~\ref{sec:death}, we defined a dead component as one receiving persistently small gradient and update signal throughout training; our analysis here, on the other hand, is post-hoc, making it more difficult to distinguish between dead neurons and neurons representing stable features. These and other issues limit our study, and so we present it merely as preliminary verification that our results are transferable to other settings.

We first analyze Rnj-1 from EssentialAI~\cite{rnj1_base} which is an 8.3B dense model trained with Muon. Rnj-1 uses GeGLU~\cite{shazeer2020gluvariants} activation which is not represented in our own experiments but fits cleanly into our more general framework. We freeze the model weights and accumulate gradient momentum for one thousand steps. Every one-hundred steps, we analyze the row-leverage of the current momentum buffer and the average activation norms over the most recent one-hundred-step window. We classify a neuron as \textit{dead with respect to a dataset} if its row leverage is less than $1\%$ of the average per-row leverage (the uniform value $n/m$) in both the up and gate momentum buffers, and its average activation RMS norm is less than $5\%$ of the layer's median. We choose these thresholds because they accurately capture dead neurons in our own experimental setting.

We find that neurons tend to satisfy either both or none of these conditions, except in later layers where we find a non-trivial percentage of neurons with low-leverage but large activation norms.
\begin{figure}[H]
    \centering
    \includegraphics[width=0.55\linewidth]{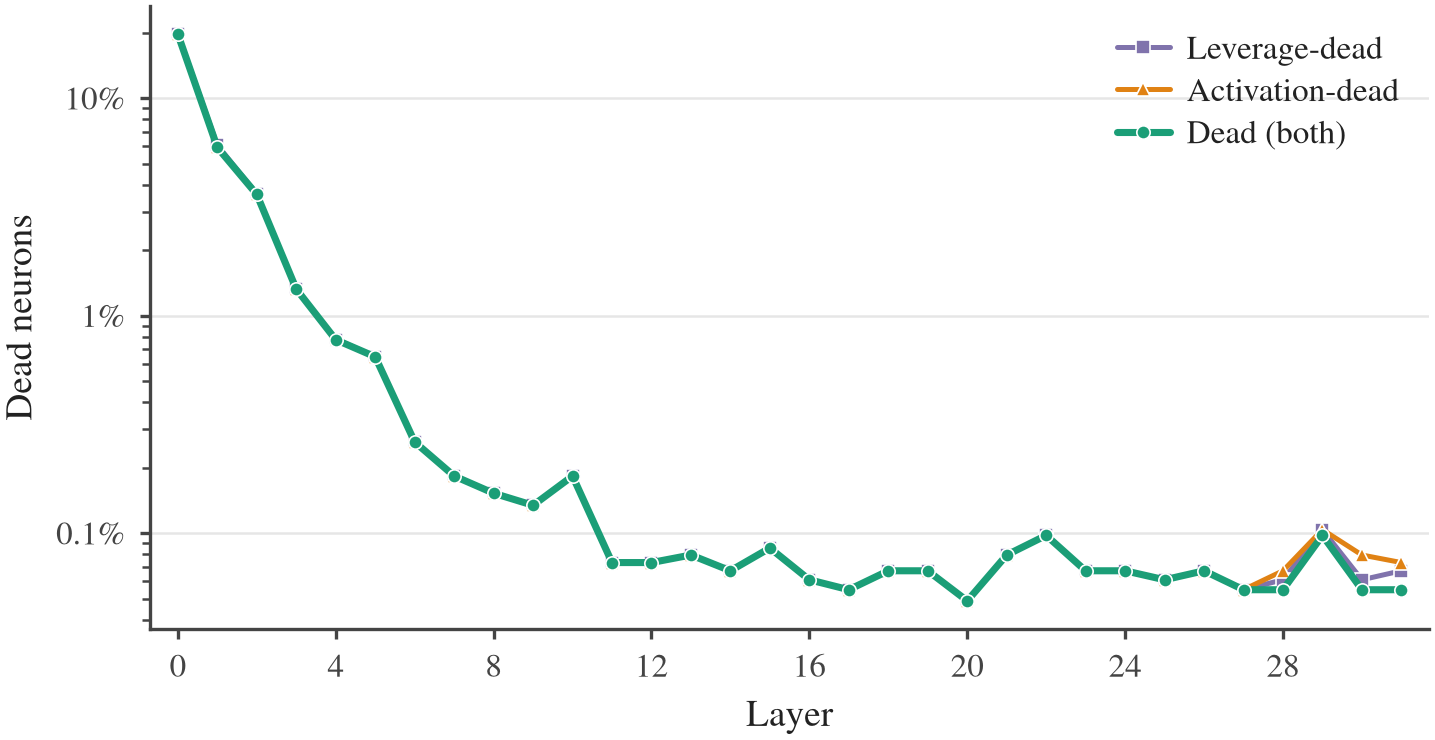}
    \caption{We find a large percentage of dead neurons in Rnj-1 on FineWeb-Edu in early layers. MLPs become more effectively utilized at greater depth, and mostly do not have dead neurons after layer ten. }
    \label{fig:rnj_dead}
\end{figure}

By our criteria, a large portion of Rnj-1's neurons are dead in early layers, but beyond layer ten, the fraction of dead neurons is consistently below $0.1\%$ of the layer's neurons. We verify this finding on a suite of open-source datasets spanning math, code and natural language.

% \begin{table}[H] % [TODO change third column]
% \centering
% \begin{tabular}{lcc}
% \toprule
% Dataset & BPB & Max \% Dead Neurons \\
% \midrule
% FineWeb-Edu & 0.689 & 20\% \\
% NemotronCCv2 & 0.639 & 20\% \\
% UltraMath & 0.215 & 20\%\\
% \bottomrule
% \end{tabular}
% \caption{Evaluating Bits per byte (BPB) and dead neurons in any layer across different open-source datasets. [TODO; this table is uninteresting and will either be deleted or remade with different info instead of col 3]
% }
% \label{tab:bpb_dead_neurons}
% \end{table}

We find that neurons that are dead on one dataset are almost always also dead on the others we test. Of the dead neurons we identified, $99.7\%$ are dead in all cases. This suggests that these neurons make an inherently small contribution to network computation, and are not just specialized to encode features that do not exist in some data types.

\begin{table}[H]
\centering
\begin{tabular}{lccc}
\toprule
Dataset & BPB & Dead neurons & Shared by all 4 \\
\midrule
FineWeb-Edu~\cite{penedo2024fineweb}  & 0.689 & 5633  & 99.8\% \\
NemotronCC v2 HQ split~\cite{nvidia_nemotron_cc_v2_2025} & 0.639 & 5628& 99.9\% \\
Nemotron-CC-Code~\cite{nvidia2025nemotroncccode} & 0.471 & 5626  & 99.9\% \\
UltraData-Math L3-QA~\cite{ultradata_math} & 0.215 & 5627 & 99.9\% \\
\midrule
\multicolumn{4}{l}{\footnotesize Total MLP neurons in Rnj-1: $524{,}288$} \\
\bottomrule
\end{tabular}
\caption{Dead-neuron counts in Rnj-1 for code, math and natural language datasets. The set of detected dead neurons is  consistent across diverse datasets; 99.8\% of dead neurons are shared across all four datasets.}
\label{tab:dead-neuron-overlap}
\end{table}

We additionally analyze the recent Laguna-XS.2 model~\cite{abadji2026laguna}, which is a 33.4B total-3B active Mixture-of-Experts (MoE) model~\cite{shazeer2017outrageously} trained with Muon. Laguna-XS.2 uses a dense MLP in its first layer and MoE layers thereafter; it is fairly common in modern MoE architectures to use an initial dense layer or intersperse dense layers at greater depths~\citep{deepseek2024v3, moonshot2026kimi26}. We choose Laguna as a recent example of a strong open-weight MoE model trained with Muon. %We focus our analysis on the dense layer, which most naturally aligns with our setting. 

% \begin{table}[H]
% \centering
% \begin{tabular}{lrr}
% \toprule
% Category & Count & \% of layer 0 \\
% \midrule
% Activation-dead & 2{,}340 & 28.56\% \\
% Leverage-dead & 2{,}374 & 28.98\% \\
% Both (act-dead $\wedge$ lev-dead) & 2{,}340 & 28.56\% \\
% Leverage-dead only & 34 & 0.42\% \\
% Activation-dead only & 0 & 0.00\% \\
% Either (union) & 2{,}374 & 28.98\% \\
% \bottomrule
% \end{tabular}
% \caption{Overlap between activation-dead and leverage-dead neurons in layer 0.}
% \label{tab:dead-neuron-overlap-layer0}
% \end{table}

\begin{table}[H]
\centering
\begin{tabular}{lcccc}
\toprule
Dataset & BPB & Dead neurons & \% of layer 0 (8192) & Shared by all 3 \\
\midrule
FineWeb-Edu              & 1.1688 & 2341 & 28.58\% & 99.8\% \\
NemotronCC v2 HQ split   & 1.1751 & 2340 & 28.56\% & 99.9\% \\
UltraData-Math L3-QA     & 0.4988 & 2338  & 28.54\% & 100.0\% \\
Nemotron-CC-Code & 0.9124 & 2340 & 28.56\% & 99.9\% \\
\midrule
\multicolumn{5}{l}{\footnotesize Total layer-0 MLP neurons in Laguna-XS.2: 8,192} \\
\bottomrule
\end{tabular}
\caption{Dead-neuron statistics for the first layer of Laguna-XS.2. BPB is within expected range and we find the same cluster of roughly $28\%$ dead neurons exists on all the datasets we test.}
\label{tab:dead_neurons_datasets}
\end{table}
The MoE layers in Laguna-XS.2 use fine-grained experts, which are wide matrix parameters and thus naturally receive uniform-leverage updates under Muon. Our previous theory predicts that these neurons should have high average utilization, and indeed we find this to be the case. We find that roughly 28\% of neurons in the initial dense layer are dead by our criteria. %We discuss fine-grained MoE layers more in Section~\ref{sec:discussion}. 

We conclude that neuron death in dense MLP layers, and more generally, non-uniform updates of tall weight matrices, may be a general problem for models trained with Muon. We suspect the percentage of dead neurons depends strongly on initialization scale and model size. %We leave an analysis of these dependencies for future work. 

\section{Aurora}\label{sec:aurora}
In Sections~\ref{sec:normalization} \&~\ref{sec:death}, we found that models trained with Muon can have a significant number of dead neurons. We showed that this pathology is effectively mitigated by applying row-normalization to Muon updates for tall matrices. We also showed that minimizing polar approximation error improves the downstream performance of Muon and aligns better with Muon's theoretical motivation (Section~\ref{sec:normalization}). We now derive Aurora as a way to more faithfully achieve both of these ideals simultaneously.

\subsection{Augmenting the Muon Objective}

Recall that the Muon update direction can be specified as the solution to the following constrained optimization problem:
\[
U^\star
=
\arg\max_U \operatorname{Tr}(G^\top U)
\qquad
\text{subject to}
\qquad
\|U\|_2 \le \eta .
\]
Here \(G = \nabla_W \mathcal{L}\) is the gradient with respect to a matrix
parameter \(W\) and $\eta$ is the learning rate. We augment the Muon objective for tall matrices
\((m > n)\) with the additional requirement that row norms be equal in the update $U$.
\begin{equation}\label{eq:dual_constraint}
U^\star
=
\arg\max_U \operatorname{Tr}(G^\top U)
\qquad
\text{s.t.}
\qquad
\|U\|_2 \le \eta,
\qquad
\|U_{i,:}\|^2 = \eta^2\frac{n}{m}
\quad
\forall i     
\end{equation}

\begin{prop}\label{claim:constraints}
Let $U \in \mathbb{R}^{m \times n}$ with $m > n$. Suppose every row has squared norm $n/m$, i.e.
\[
    \|U_{i,:}\|_2^2 = \frac{n}{m}
    \qquad \forall i \in [m].
\]
Then $\|U\|_2 \le 1$ if and only if $U^\top U = I_n$. Equivalently, under the uniform row-norm constraint, the spectral norm relaxation
$\|U\|_2 \le 1$ is tight and forces $U$ to have orthonormal columns.
\end{prop}

Thus, by Proposition~\ref{claim:constraints}, the objective in Equation~\ref{eq:dual_constraint} can be rewritten as 
\begin{equation}\label{eq:dual_constraint2}
U^\star
=
\arg\max_U \operatorname{Tr}(G^\top U)
\qquad
\text{s.t.}
\qquad
U^{\top}U = \eta^2 I
\qquad
\|U_{i,:}\|^2_2 = \eta^2\frac{n}{m}
\quad
\forall i     
\end{equation}
The solution is steepest descent under a left semi-orthogonality constraint to make the update RMS-to-RMS norm preserving, and a constant row norm constraint to enforce uniform updates. We derive Aurora as a solution to this problem.

\subsection{Riemannian-Aurora}
We first consider a Riemannian formulation of the dual-constraint objective (Equation~\ref{eq:dual_constraint2}). Instead of enforcing the spectral-norm and leverage constraints separately, we constrain the update $U$ to (approximately) lie on the equal-row-leverage Stiefel manifold.

\[
\mathcal{M}_{m,n}
=
\left\{
U \in \mathbb{R}^{m \times n}
:
U^\top U = \eta^2 I_n,
\quad
\operatorname{diag}(UU^\top)=\eta^2 \frac{n}{m}\mathbf{1}
\right\}.
\]

In our convergence analysis (Theorem~\ref{thm:aurora_converge} and Appendix~\ref{app:riemannian}) we set the step size $\eta=1$ without loss of generality, so that $r=n/m$; the general case simply rescales these constraints by $\eta$. The resulting Riemannian approach projects the Euclidean direction \(G\) onto the tangent space of \(\mathcal{M}\), takes a
step along the projected direction, and then retracts the result back towards 
\(\mathcal{M}_{m, n}\). We present the full derivation for Riemannian-Aurora in Appendix~\ref{app:riemannian}.

\begin{algorithm}[H]
\caption{Riemannian-Aurora}
\label{alg:riemannian_aurora}
\begin{algorithmic}[1]
\Require Momentum \(G \in \mathbb{R}^{m \times n}\), step size \(\eta\), iterations \(T\)
\State Set \(r \gets \eta^2 \frac{n}{m}\)
\State Initialize \(U_0 \in \mathcal{M}\) 
\Comment{e.g. polar projection of \(G\) followed by row/Stiefel projection}
\For{\(t = 0,1,\ldots,T-1\)}
    \State \(B_t \gets \operatorname{sym}(U_t^\top G)\)
    \State \(P_t \gets U_t U_t^\top\)
    \State \(q_t \gets \operatorname{diag}(G U_t^\top) - \operatorname{diag}(U_t B_t U_t^\top)\)
    \State Solve \((r I_m - P_t \odot P_t)\lambda_t = q_t\) \label{step:conjugate_grad}
    \State \(D_t \gets \operatorname{diag}(\lambda_t)\)
    \State \(S_t \gets B_t - U_t^\top D_t U_t\)
    \State \(Z_t \gets G - U_t S_t - D_t U_t\)
    \State \(Y_t \gets U_t + \eta Z_t\)
    \State \(U_{t+1} \gets \operatorname{Retract}(Y_t)\)
\EndFor
\State \Return \(U_T\)
\end{algorithmic}
\end{algorithm}

Unlike post-hoc row normalization, which first computes the polar factor and then
rescales its rows, Riemannian-Aurora treats orthogonality and equal row leverage
as a single joint constraint. Both the initialization of \(U_0\) and retraction
are done approximately via alternately computing the polar factor and row
normalizing the update in our implementation.

Some of the steps in Riemannian-Aurora are particularly expensive, making it
infeasible for use on moderately-sized networks as written. The main expensive operations are the following:
\begin{enumerate}
    \item \textbf{Forming the leverage constraint operator.} The tangent space projection involves solving an \(m\)-dimensional
    linear system whose coefficient matrix involves \(r I - P \odot P\), where \(P = UU^\top\). For
    \(U \in \mathbb{R}^{m \times n}\), explicitly forming \(P\) costs \(O(m^2 n)\)
    operations and storing it costs \(O(m^2)\) memory, which is prohibitive at even moderate scales.

    \item \textbf{Solving for the Lagrange multiplier vector in Line~\ref{step:conjugate_grad}}. The most common way to solve the system in Line~\ref{step:conjugate_grad} is with conjugate gradient (CG), which contains repeated applications of the map $v \mapsto (r I - P \odot P)v$. The term $(P \odot P)v$ couples each row with every other row in $P$, resulting in dense row interactions. Thus, each CG iteration is expensive even when $P$ is not explicitly factorized.
 \end{enumerate}
Due to these limitations, we use Riemannian-Aurora purely as a reference solver against which to compare our more practical instantiation of Aurora.

%Our more practical instantiation of Aurora instead uses an iterative solution, which alternates 
%We present it mostly as a faithful solution to our dual-constraint problem and
%as a comparison point for our practical instantiation of Aurora. We include a
%small-scale verification of Riemannian-Aurora in Appendix~E.

\subsection{Aurora: An Iterative Solution}
Though the joint Stiefel/equal-row-leverage projection has no simple closed form,
it admits a natural iterative approximation which alternates between the two
constraint sets, starting with row normalization, then using an NS iteration to approximate the polar factor. We show that, under a local regularity condition, this iteration converges linearly to an update satisfying both the Stiefel and equal row-norm constraints.

\begin{theorem}[Local linear convergence of Aurora]\label{thm:aurora_converge}
Let $m>n$ and $r=n/m$, and define
\[
\mathcal S
=
\{U\in\mathbb R^{m\times n}: U^\top U=I_n\},
\qquad
\mathcal O_r
=
\{U\in\mathbb R^{m\times n}: \|U_{i,:}\|_2^2=r,\ i=1,\ldots,m\}.
\]
Let $\mathcal M=\mathcal S\cap \mathcal O_r$ and suppose
$U_\star\in\mathcal M$ is regular, meaning that for
$P_\star=U_\star U_\star^\top$,
\[
\ker\!\left(rI_m - P_\star\odot P_\star\right)
=
\operatorname{span}\{\mathbf 1\}.
\]
Then, for every initialization $U_0$ sufficiently close to $U_\star$, the
alternating projection sequence
\[
U_{k+1/2}=\Pi_{\mathcal O_r}(U_k),
\qquad
U_{k+1}=\Pi_{\mathcal S}(U_{k+1/2})
\]
is well-defined and converges linearly to some point
$U_\infty\in\mathcal M$.
\end{theorem}

Intuitively, the regularity condition in Theorem~\ref{thm:aurora_converge} precludes degenerate intersections of the Stiefel and equal row-norm constraints. Under this assumption, Aurora is locally linearly convergent. We will also validate Aurora's convergence empirically in our pre-training settings (Section~\ref{sec:training}).

\begin{algorithm}[H]
\caption{Aurora}
\label{alg:aurora}
\begin{algorithmic}[1]
\Require Momentum buffer \(M \in \mathbb{R}^{m \times n}\), NS iterations \(K\), damping parameter \(\beta \in (0, 1)\).
\State \(X_0 \gets M / \|M\|_F\)
\State \(D_{-1} \gets I_m\)
\For{\(k = 0,1,\ldots,K-1\)}
    \State \(r_k \gets \left(\|(X_k)_{1,:}\|,\ldots,\|(X_k)_{m,:}\|\right)\)
    \State \(D_k \gets D_{k-1}^{\beta} \cdot \operatorname{diag}(r_k)^{1-\beta}\)
    \State \(\widetilde{X}_k \gets \sqrt{n/m}\, D_k^{-1} X_k\)
    \State \(X_{k+1} \gets \operatorname{polar}(\widetilde{X}_k)\)\label{step:stiefel}
\EndFor
\State \Return \(X_K\)
\end{algorithmic}
\end{algorithm}

Aurora uses a damped iteration, which we find practically useful for improving its convergence. This requires just one additional scalar parameter \(\beta\). Additionally, we note that the last step of the iteration is always a polar factor computation (Step~\ref{step:stiefel}) so that the updates produced by Aurora will (approximately) lie on the Stiefel manifold.

Aurora is roughly \(K\) times more expensive than Muon, where $K$ is the number of inner-loop steps. However, in an optimized distributed training setup, this additional computation can often be effectively overlapped with gradient communication, so the wall-clock overhead is fairly modest. We evaluate the throughput degradation on a single H100 and in our distributed setting in Section~\ref{sec:training}.

% \begin{figure}[H]
%     \centering
%     \includegraphics[width=0.8\linewidth]{new_figs/error.png}
%     \caption{Final row-norm spread for Aurora and Muon on random $512\times128$ tall matrices. We sweep the initial row-norm spread $\sigma$ and spectral anisotropy $\alpha$, and report the coefficient of variation (CV) of the final row norms. Aurora, with $K=3$ and $\beta=0.8$, maintains substantially more uniform row norms than a single Muon polar step across the full sweep.}
%     \label{fig:row_spread}
% \end{figure}

\section{Experiments}\label{sec:training}
%We empirically evaluate Aurora on synthetic tasks and language model pre-training. We compare against MuonEq-R and NorMuon as well as vanilla Muon all using CANS-12. 

%\subsection{Language Modelling Benchmarks}\label{sec:pretrain}

\textbf{Experiment setup.} For all experiments, we use a transformer architecture with gated attention, RMSNorm on the query and key vectors and short convolution in the attention layers. All models are pretrained on 100B tokens from the high-quality split of the NemotronCC v2 dataset with a context length of 4096. We compare Aurora against Muon, U-NorMuon and NorMuon at 340M and 1.1B parameter scales\footnote{Model parameter counts do not include embedding parameters and we use tied embeddings.}. In all cases, we use a batch size of 4M tokens and we use our own tokenizer which uses a vocabulary size of 128K.

% \begin{table}[H]
% \centering
% \begin{tabular}{lccccc}
% \toprule
% Models & Layers & Hidden Size & GQA & Positional Encoding & Additional Features \\
% \midrule
% Aurora-340M & 24 & 1024 & 4:1 & RoPE & QKNorm, ShortConv, Gated Attention \\
% Aurora-1.1B & 24 & 2048 & 4:1 & RoPE & QKNorm, ShortConv, Gated Attention \\
% \bottomrule
% \end{tabular}
% %\caption{We use a .}
% \label{tab:model_architectures}
% \end{table}

\textbf{Evaluation and benchmarks}. We evaluate on the following general knowledge and reasoning tasks. MMLU~\cite{hendrycks2021mmlu} (5-shot), HellaSwag (HSwag)~\cite{zellers2019hellaswag} (10-shot), PIQA~\cite{bisk2020piqa} (0-shot), WinoGrande (Wino.)~\citep{sakaguchi2020winogrande} (5-shot), ARC-Challenge (ARC-C)~\cite{clark2018arc} (25-shot) and ARC-Easy (ARC-E), LAMBADA (LMB)~\citep{paperno2016lambada} (0-shot), and OpenBookQA (OBQA)~\citep{mihaylov2018openbookqa} (0-shot).

\begin{table}[H]
\centering
\resizebox{\textwidth}{!}{
\begin{tabular}{llccccccccc}
\toprule
Model Size & Optimizer
& \begin{tabular}{c}ARC-C\\acc$\uparrow$\end{tabular}
& \begin{tabular}{c}ARC-E\\acc$\uparrow$\end{tabular}
& \begin{tabular}{c}HSwag\\acc$\uparrow$\end{tabular}
& \begin{tabular}{c}LMB.\\acc$\uparrow$\end{tabular}
& \begin{tabular}{c}MMLU\\acc$\uparrow$\end{tabular}
& \begin{tabular}{c}OBQA\\acc$\uparrow$\end{tabular}
& \begin{tabular}{c}PIQA\\acc$\uparrow$\end{tabular}
& \begin{tabular}{c}Wino.\\acc$\uparrow$\end{tabular}
& \begin{tabular}{c}Avg\\$\uparrow$\end{tabular} \\
\midrule
\multirow{4}{*}{340M}
& Muon      & 33.6 & \textbf{60.2} & 49.1 & 40.2 & 24.7 & 34.6 & 70.7 & 51.7 & 45.6 \\
& NorMuon   & 34.6 & 56.9 & 49.4 & 39.1 & \textbf{25.7} & 34.2 & 70.4 & \textbf{54.9} & 45.7 \\
& U-NorMuon & 34.0 & 56.6 & 49.9 & 41.5 & 25.4 & 33.4 & 70.9 & \textbf{54.9} & 45.8 \\
& Aurora    & \textbf{35.3} & 56.8 & \textbf{50.7} & \textbf{43.5} & 24.0 & \textbf{35.6} & \textbf{71.2} & 54.4 & \textbf{46.4} \\
\midrule
\multirow{4}{*}{1.1B}
& NorMuon   & 45.7 & 60.9 & \textbf{68.6} & 55.1 & 30.1 & 39.4 & 76.2 & 64.8 & 55.1 \\
& U-NorMuon & 43.6 & 67.4 & 67.4 & \textbf{56.9} & 29.1 & \textbf{40.6} & \textbf{77.1} & 61.9 & 55.5 \\
& Muon      & 46.0 & \textbf{70.6} & 67.2 & 55.7 & 29.7 & 39.2 & 76.3 & 63.1 & 56.0 \\
& Aurora    & \textbf{46.9} & 67.9 & 68.4 & 54.5 & \textbf{38.8} & 39.2 & 76.9 & \textbf{64.9} & \textbf{57.2} \\
\bottomrule
\end{tabular}
}
\caption{Comparison of Muon, NorMuon, U-NorMuon and Aurora on general understanding evaluations. The best result for each model scale is bolded. Aurora achieves the strongest average score at both scales with an MMLU score 9.1 points higher than Muon at the 1.1B scale. These results differ from those reported in the original blog post due to improved hyperparameter settings.}
\label{tab:downstream_eval}
\end{table}

% \begin{figure}[H]
%     \centering
%     \includegraphics[width=0.8\linewidth]{figures/1B_pretraining.png}
%     \caption{1.1B pretraining loss. Aurora (green) converges faster and reaches a lower final loss than all Muon and NorMuon baselines. All runs train for 24k steps.}
%     \label{fig:1B_pretraining}
% \end{figure}

\paragraph{Throughput Comparison.} We report single-GPU throughput to isolate the local optimizer overhead of Aurora relative to Muon. This overhead reflects the additional computation and memory traffic from Aurora's row-balancing steps. In a distributed setting, we find that much of this overhead can be overlapped with gradient communication, making the wall-clock overhead due to Aurora relatively small. %We report throughput figures from our 4-node 1.1B training runs.

\begin{table}[H]
\centering
\sisetup{
  detect-weight=true,
  detect-inline-weight=math,
  table-number-alignment=center
}
\begin{tabular}{llS[table-format=2.1]S[table-format=2.1]}
\toprule
\multirow{2}{*}{Model Size}
& \multirow{2}{*}{Optimizer}
& \multicolumn{2}{c}{Throughput (k tok/s/GPU)} \\
\cmidrule(lr){3-4}
& & {1 GPU} & {8 GPUs} \\
\midrule
\multirow{3}{*}{340M}
  & Muon          & 56.3   & 88.2   \\
  & Aurora ($K=2$) & 47.3   & 86.5   \\
  & Aurora ($K=3$) & 41.8   & 86.1   \\
\midrule
\multirow{3}{*}{1.1B}
  & Muon          & 23.4   & 36.1 \\
  & Aurora ($K=2$) & 17.7   & 35.8 \\
  & Aurora ($K=3$) & 14.7   & 35.5 \\
\bottomrule
\end{tabular}
\caption{Training throughput for Muon and Aurora variants on a single H100 GPU and in our one-node training setting. All experiments use a sequence length of 2048 tokens and a 4M-token batch size. Throughput figures are averages across the training run after an initial warm-up period of 50 steps.}
\label{tab:throughput_single_gpu_one_node}
\end{table}

% on our 340M and 1.1B transformer configurations on a single H100 to measure the per-device overhead due to Aurora. In an optimized distributed implementation, much of this additional compute 

%We report single-GPU throughput to isolate the local optimizer overhead of Aurora relative to Muon. This overhead reflects the additional computation and memory traffic from Aurora's row-balancing steps, independent of distributed communication effects.

%\subsection{\texttt{modded-NanoGPT} and Synthetic Results}
%[TODO: I don't think this really makes sense because we're marketing aurora as for up/gate proj specifically.]
% We train a 2-layer MLP with ReLU for three epochs on CIFAR-10 as a small-scale verification of Riemannian-Aurora.  Importantly, we apply 
% We find that Riemannian-Aurora significantly outperforms the tuned AdamW baseline and performs comparably to Muon. Unlike Muon, it leads to much more uniform row leverage scores [TODO: need a leverage scores plot alongisde Fig 8].

% \begin{figure}[H]
%     \centering
%     \includegraphics[width=0.7\linewidth]{figures/cifar10.png}
%     \caption{For Riemannian-Aurora hyperparameters, we use $K=20$ conjugate gradient iterations, Nesterov Momentum, lr$=4\mathrm{e}{-3}$, 2 retraction steps, a Riemannian learning rate of $\eta=0.1$. We AdamW, we use a learning rate of $4\mathrm{e}{-3}$ with betas $(0.9, 0.95)$. We found these parameters to be approximately optimal after a hyperparameter sweep.}
%     \label{fig:placeholder}
% \end{figure}

\textbf{Aurora Convergence.} To evaluate the accuracy of the Aurora projection, we sample momentum buffers at different points in our 1.1B Muon training run and plot the precision of the resulting updates. Both Aurora and Riemannian-Aurora produce very accurate solutions to the dual-constraint problem after just three iterations, but the Riemannian solver tends to produce updates more aligned with the raw gradient direction.

\begin{figure}[H]
    \centering
    \includegraphics[width=0.5\linewidth]{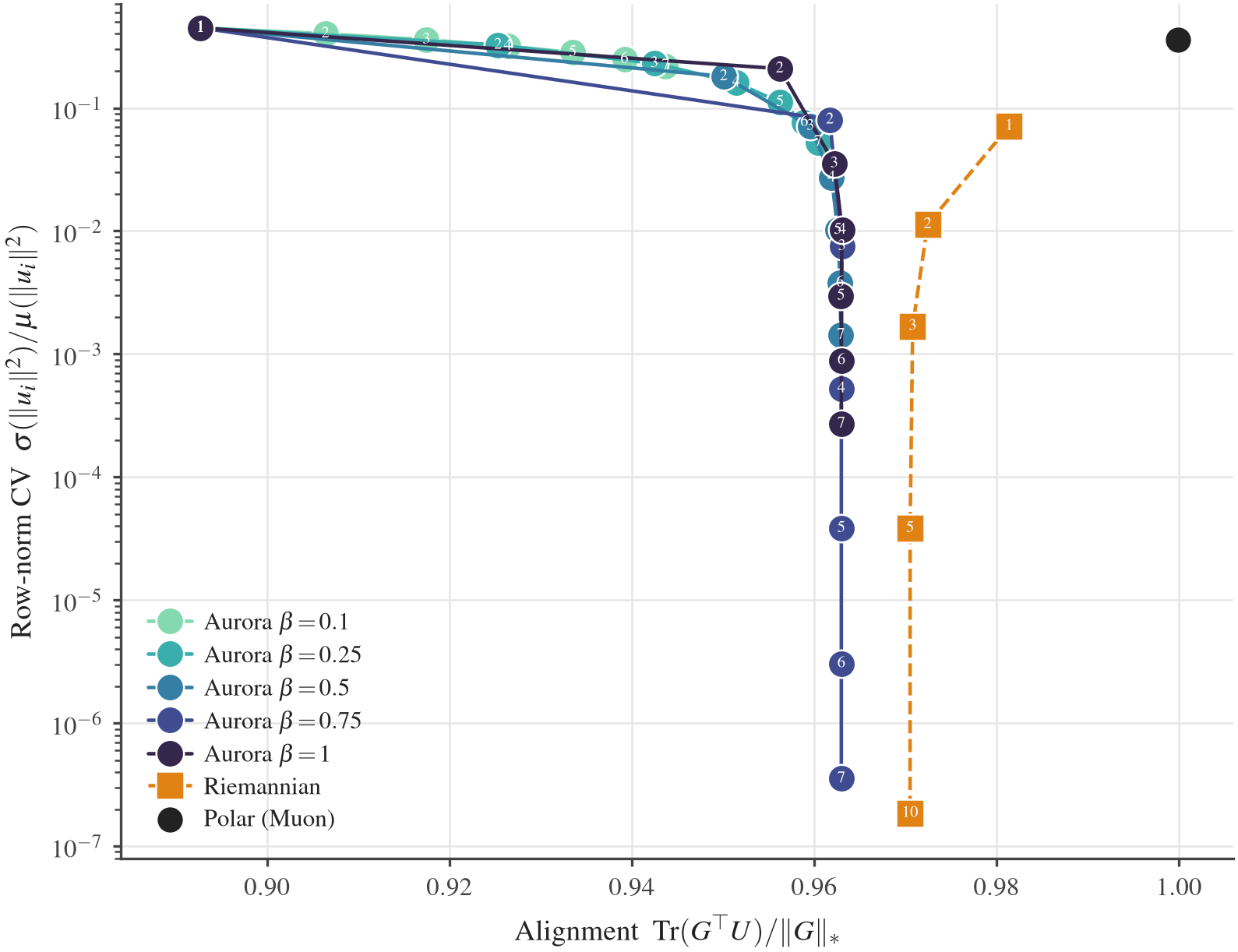}
    \caption{Comparison of the alignment and row-norm coefficient-of-variation trade-off for Aurora, Riemannian-Aurora, and standard Muon. All updates lie on the Stiefel manifold by construction. Data is computed on a real momentum buffer from our 1.1B Muon training run at step 100.}
    \label{fig:solver}
\end{figure}

\textbf{\texttt{modded-nanoGPT} Optimizers Track.} We additionally evaluate Aurora on the \texttt{modded-nanoGPT} speedrun benchmark (Track 3), comparing against the current state-of-the-art and the vanilla NorMuon baseline. We find that vanilla Aurora ($K=2$, $\beta=0.5$) improves upon the vanilla NorMuon baseline, reaching the target validation loss in 25 fewer steps. Combining Aurora with Contra-Muon~\cite{nilin2026contramuon} and update/weight flooring yielded a new state-of-the-art result at the time of writing.

\begin{figure}[H]
    \centering
    \includegraphics[width=0.6\linewidth]{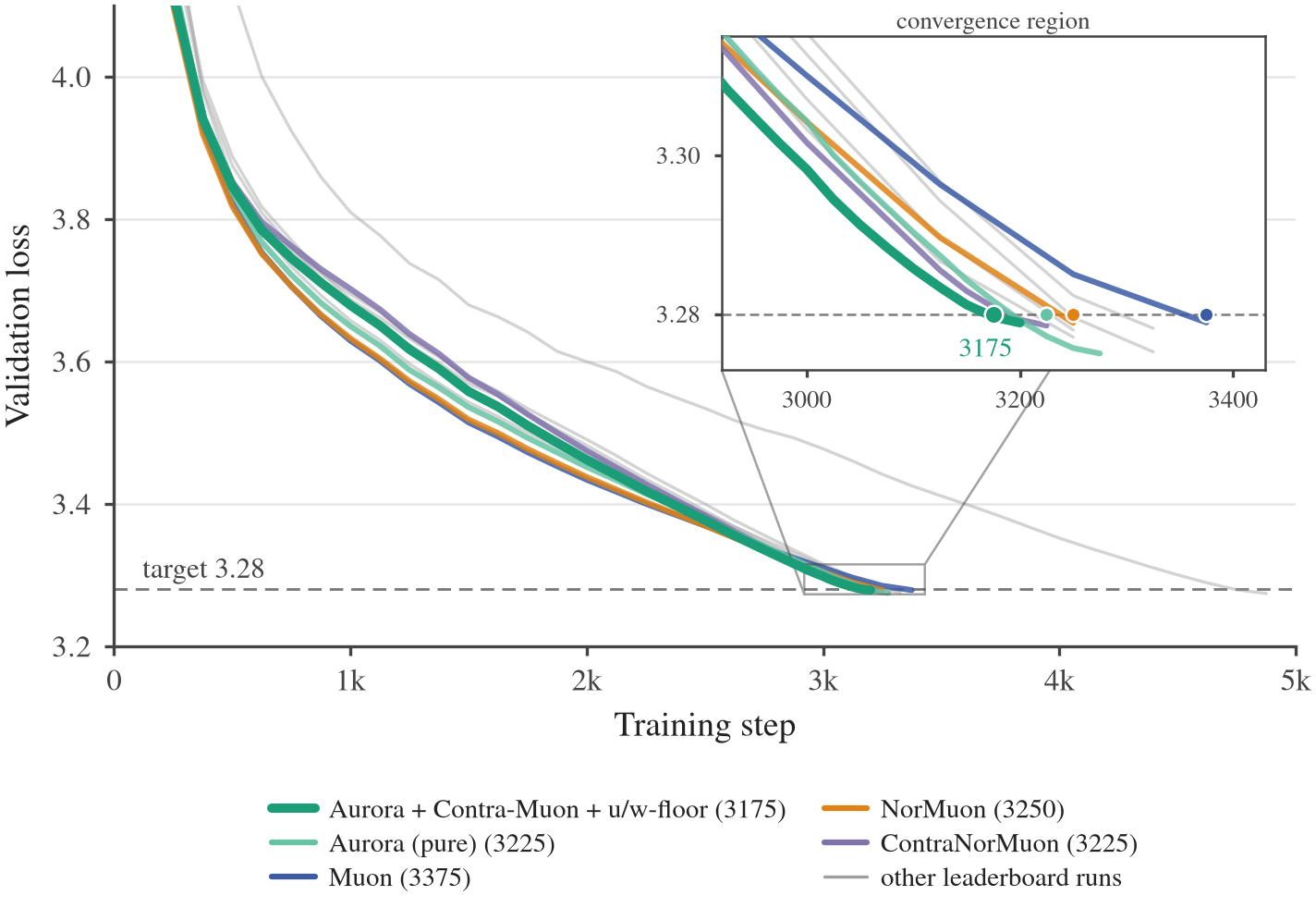}
    \caption{\texttt{modded-nanoGPT} speedrun convergence curves. Aurora and Contra-Muon reaches the 3.28 validation loss target at step 3175, setting a new state-of-the-art on the optimization track.}
    \label{fig:nanogpt}
\end{figure}

\paragraph{Training wide MLPs.} We experiment with training MLPs of different widths using both Muon and Aurora. The base architecture is our 340M transformer, which uses an MLP expansion factor of four; all other training and architecture settings are held fixed across runs. We find that the advantage of Aurora over Muon grows monotonically with MLP expansion factor (Figure~\ref{fig:mlp_exp}), suggesting that Aurora is particularly effective for wide MLPs.

\begin{figure}[H]
    \centering
    \includegraphics[width=0.6\linewidth]{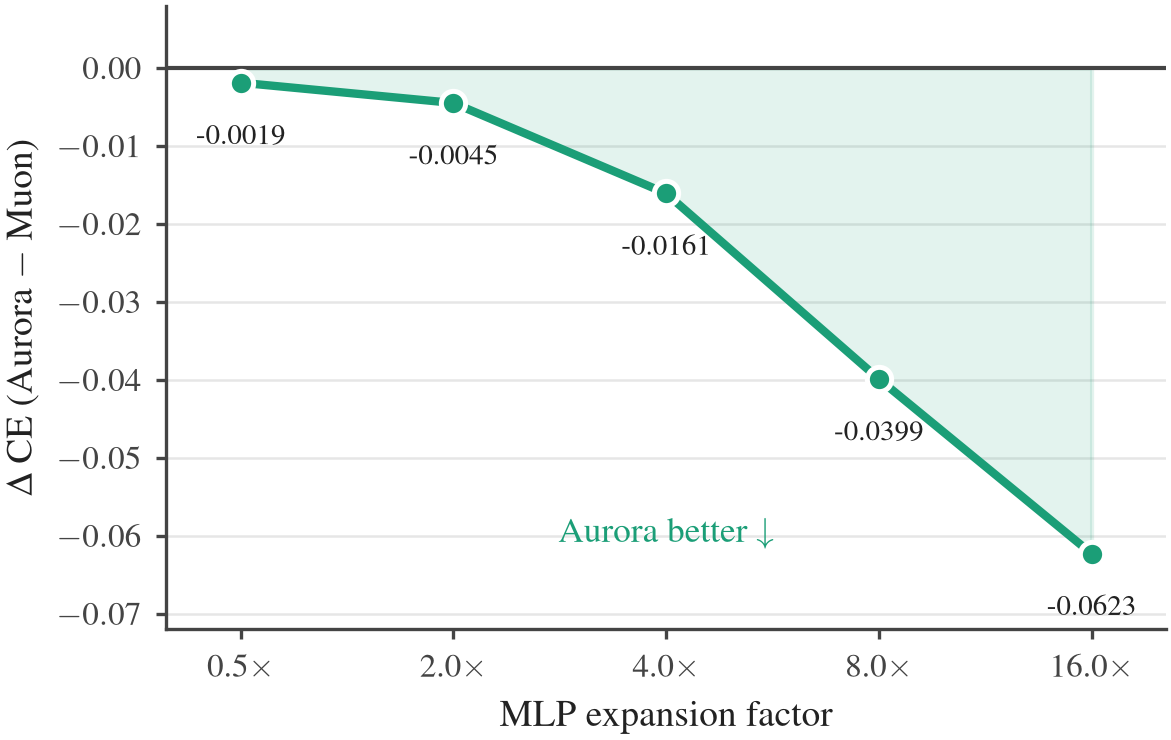}
    \caption{Aurora's advantage over Muon scales with MLP expansion factor. We train a model with each expansion factor using both Muon and Aurora, plotting the gap in final smoothed loss. We stress that each of these models has a different parameter count.}%, and our comparison is between models of the same scale.}
    \label{fig:mlp_exp}
\end{figure}

\section{Discussion and Future Work}\label{sec:discussion}
We derived Aurora to resolve a specific pathology in language model training which we first observed empirically and then understood mechanistically. We showed that Aurora effectively prevents neuron death while still adhering faithfully to Muon's theoretical framing. We believe this type of \textit{bottom-up} optimizer design, whereby algorithms are derived or modified based on observations from empirical training dynamics, is an interesting direction to pursue more generally. Many existing optimizers are motivated purely by theoretical models which may not accurately capture real language model training dynamics. 

Additionally, Aurora was designed to optimize specifically the tall matrices in the model, such as the up and gate projections. We believe optimizer design at this level of granularity is a  promising emerging area of research, with clear advantages over the standard ``one-size-fits-all'' strategy. This \textit{architecture-optimizer codesign} incorporates the observation that different classes of parameters can play very different roles in network computation, and should receive updates that are aware of this fact.

We stress that Aurora's modification to Muon can be implemented so that the additional computation provides very little wall-clock overhead. We believe there are a number of routes to improving Aurora, including improved solutions for the row-oblique Stiefel projection problem, and various practical performance enhancements. For example, we observed that under Aurora, the momentum buffer tends towards a stable region in which its rows have nearly equal row leverage. It may be possible to (adaptively) switch to Muon when training has reached this stable regime. We leave these as directions for future research.

\bibliographystyle{plainnat}
\bibliography{biblio}

\appendix

% \section{340M Training Settings}\label{app:340m}
% We use the following settings for our 340M experiments with Muon.
% \begin{table}[H]
% \centering
% \begin{tabular}{p{0.28\linewidth}p{0.65\linewidth}}
% \toprule
% Category & Details \\
% \midrule
% Training Configuration 
% & 800k tokens per batch, 2048 token sequence length, WSD schedule \\

% Data 
% & 20B tokens of NemotronCC v2 from the high quality split \\

% Hyperparameters 
% & \(\mathrm{lr}=8\mathrm{e}{-3}\) \\

% Model Architecture 
% & \texttt{hidden\_size}=1024, layers=24, GQA 4:1, RoPE, QKNorm, ShortConv, Gated Attention \\
% \bottomrule
% \end{tabular}
% \end{table}

\section{U-NorMuon Ablation Study}\label{app:u_normuon}
We ablate all the modifications U-NorMuon applies onto NorMuon individually. We use the term \textit{stateful} to refer to variants that maintain a buffer of row norms as in NorMuon. Variants using polar scaling only do so for the row-normalized parameters; all other non-embedding matrix parameters receive spectral scaling in all cases. We tune two sets of hyperparameters: one for the polar-scaled variants and another for the spectral-scaled optimizers. 

% \begin{table}[H]
% \centering
% \begin{tabular}{lccccc}
% \toprule
% Variant
% & Stateful
% & Scaling
% & Parameters
% & Validation Loss $\downarrow$
% & $\Delta$ loss $\uparrow$
% \\
% \midrule
% Muon 
% & --
% & -- 
% & --
% & 2.591
% & --
% \\
% Polar, stateless, all
% & \xmark
% & Polar
% & All
% & 2.568
% & --
% \\
% Polar, stateful, up/gate
% & \cmark
% & Polar
% & Up/gate
% & 2.561
% & --
% \\
% Polar, stateful, all
% & \cmark
% & Polar
% & All
% & 2.577
% & --
% \\
% Spectral, stateless, up/gate
% & \xmark
% & Spectral
% & Up/gate
% & 2.563
% & --
% \\
% Spectral, stateless, all
% & \xmark
% & Spectral
% & All
% & 2.575
% & --
% \\
% Stateful and spectral
% & \cmark
% & Spectral
% & Up/gate
% & 2.569
% & 0.002
% \\
% Spectral, stateful, all
% & \cmark
% & Spectral
% & All
% & 2.564
% & --
% \\
% U-NorMuon
% & \xmark
% & Polar
% & Up/gate
% & \textbf{2.567}
% & \textbf{0.000}
% \\
% \bottomrule
% \end{tabular}
% \caption{
% Ablation over the design axes. $\Delta$ loss is computed relative to U-NorMuon, so positive values indicate worse final loss.
% }
% \label{tab:four_axis_ablation}
% \end{table}

\begin{table}[H]
\centering
\begin{tabular}{lccccc}
\toprule
Variant & Stateful & Scaling & Parameters & Validation Loss $\downarrow$ & $\Delta$ loss vs. Muon \\
\midrule
Muon                         & --           & Spectral & All     & 2.591          & 0.000  \\
NorMuon                      & \checkmark   & Spectral & All     & 2.577          & -0.014 \\
Polar, stateful, all          & \checkmark   & Polar    & All     & 2.577          & -0.014 \\
Spectral, stateless, all      & \xmark       & Spectral & All     & 2.575          & -0.016 \\
Spectral, stateful, up/gate         & \checkmark   & Spectral & Up/gate & 2.569          & -0.022 \\
Polar, stateless, all         & \xmark       & Polar    & All     & 2.568          & -0.023 \\
Spectral, stateful, all       & \checkmark   & Spectral & All     & 2.564          & -0.027 \\
Spectral, stateless, up/gate  & \xmark       & Spectral & Up/gate & 2.563          & -0.028 \\
Polar, stateful, up/gate     & \checkmark       & Polar    & Up/gate & 2.562          & -0.029 \\
U-NorMuon                     & \xmark   & Polar    & Up/gate & \textbf{2.561} & \textbf{-0.030} \\
\bottomrule
\end{tabular}
\caption{Ablation over U-NorMuon design axes at 340M scale. Loss differences are computed relative to Muon.}
\label{tab:ablation_design_axes}
\end{table}

In all cases, we use a smoothed validation loss which we find to be very predictive of benchmark scores at larger scales. U-NorMuon outperforms all the other algorithms we ablate. U-NorMuon does not maintain row magnitude state, only row-normalizes tall matrices, and uses polar scaling on the tall matrix branch.

Row-normalizing only the up and gate projections yielded the largest improvement. Using polar scaling in NorMuon (i.e. polar scaling and row normalization on all parameters, maintaining row-norm states) did not improve final validation loss, but polar scaling led to a fairly large improvement in the stateless version of NorMuon. Maintaining row-norm state was generally not helpful except when doing spectral scaling and normalizing all parameters.

\section{Analysis of Networks Using ReLU$^2$}\label{app:relu}
We repeat our 340M experiments from Section~\ref{sec:death} using a network with the ReLU$^2$ activation for MLPs. The results closely mirror those from our 340M experiments with SwiGLU MLPs. In particular, we find consistently uniform leverage scores under U-NorMuon and a bimodal leverage distribution under Muon across steps. 

\begin{figure}[H]
    \centering
    \includegraphics[width=0.7\linewidth]{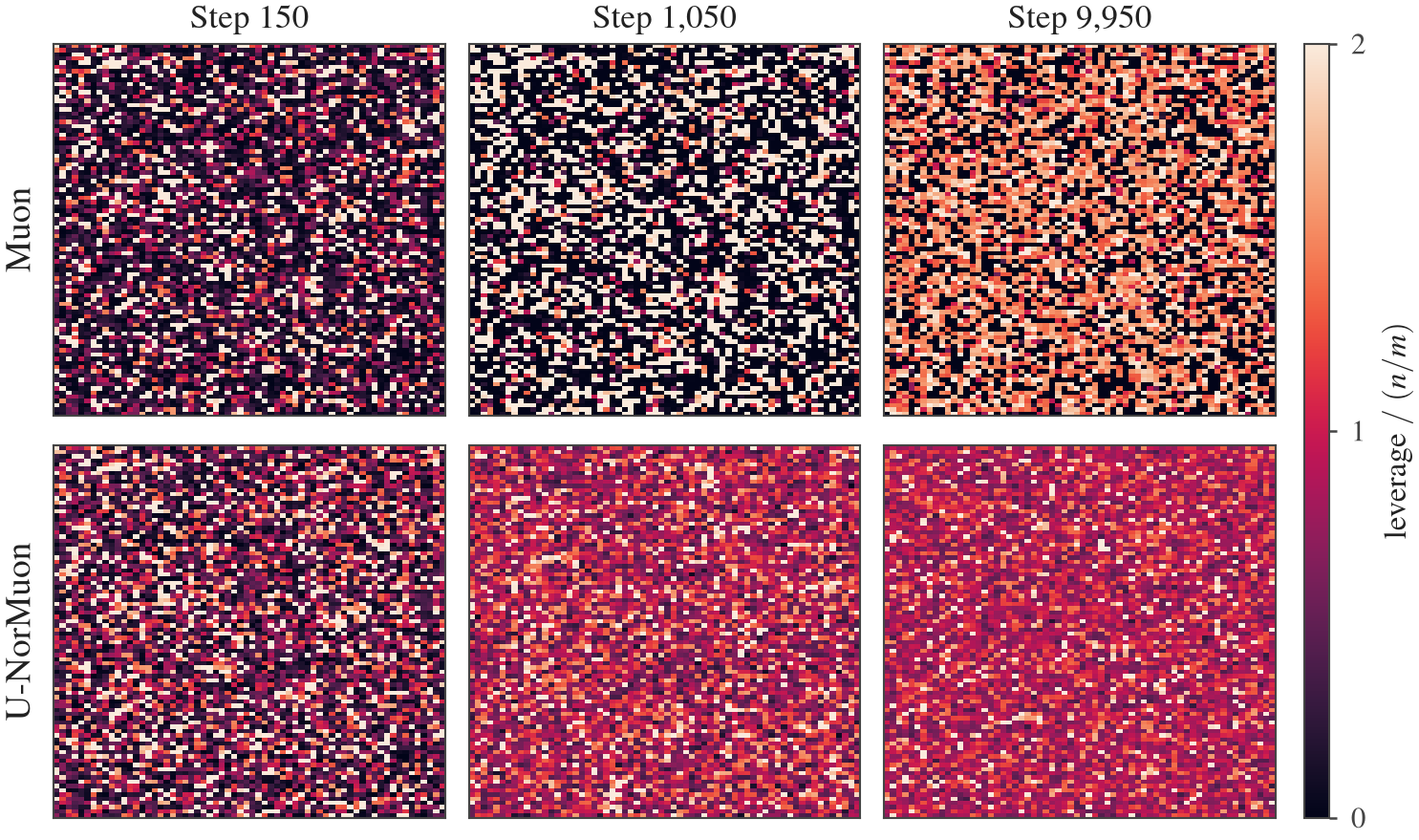}
    \caption{Leverage scores for all 5632 neurons (up projection, layer 12) in a ReLU$^2$ MLP from our 1.1B Muon and U-NorMuon runs. The same pattern of early, permanent death under Muon and healthy leverage under U-NorMuon is observed.}
    \label{fig:relu2_grid}
\end{figure}
As before, we see uniform leverage scores under U-NorMuon and a large cohort of dead neurons under Muon. Additionally, we find leverage scores persist throughout training, with the lowest-leverage neurons collapsing and never recovering. These results give empirical evidence that our main results hold for the general classes of gated MLPs, wherein coordinate-wise nonlinearities lead to momentum row norm anisotropy. 

\section{Derivation of Riemannian-Aurora}\label{app:riemannian}
We want to solve
\[
\max_{U \in \mathbb R^{m \times n}} \langle G, U\rangle
\qquad
\text{s.t.}
\qquad
U^\top U = I_n,
\qquad
\operatorname{diag}(UU^\top) = \frac{n}{m}\mathbf 1 .
\]
Let \(r := n/m\). The feasible set is the equal-row-norm Stiefel manifold
\[
\mathcal M
=
\left\{
U \in \mathbb R^{m \times n}
:
U^\top U = I_n,\;
\|U_{i,:}\|_2^2 = r \;\; \forall i
\right\}.
\]

\paragraph{Tangent space.}
Let \(Z \in \mathbb R^{m \times n}\) be a tangent vector at \(U \in \mathcal M\). The Stiefel constraint requires
\((U+\epsilon Z)^\top (U+\epsilon Z)=I_n+O(\epsilon^2)\) for small \(\epsilon>0\). Keeping only first-order terms gives
\[
U^\top Z + Z^\top U = 0.
\]
The row-norm constraint requires \(\|U_{i,:}+\epsilon Z_{i,:}\|_2^2=r+O(\epsilon^2)\) for all \(i\). Expanding to first order gives \(\langle U_{i,:},Z_{i,:}\rangle=0\) for all \(i\), or equivalently $\operatorname{diag}(ZU^\top)=0$. Therefore,
\[
\boxed{
T_U\mathcal M
=
\left\{
Z \in \mathbb R^{m \times n}
:
U^\top Z + Z^\top U = 0,\;
\operatorname{diag}(ZU^\top)=0
\right\}
}
\]

\paragraph{Riemannian gradient.}
We compute the Riemannian gradient by projecting \(G\) onto the tangent space:
\[
Z
=
\arg\min_{\widetilde Z}
\|\widetilde Z - G\|_F^2
\quad
\text{s.t.}
\quad
U^\top \widetilde Z + \widetilde Z^\top U = 0,
\qquad
\operatorname{diag}(\widetilde Z U^\top)=0.
\]

The projected vector has the form \(Z=G-US-DU\), where \(S=S^\top\in\mathbb R^{n\times n}\) and \(D=\operatorname{diag}(\lambda)\in\mathbb R^{m\times m}\). The \(US\) term enforces the Stiefel tangent constraint, while the \(DU\) term enforces the row-norm tangent constraint. Define
\[
B
:=
\operatorname{sym}(U^\top G)
=
\frac{1}{2}(U^\top G + G^\top U),
\qquad
P := UU^\top .
\]

We impose the Stiefel tangent constraint \(U^\top Z+Z^\top U=0\). Substituting \(Z=G-US-DU\), we get
\[
U^\top(G-US-DU) + (G-US-DU)^\top U = 0.
\]
Using \(U^\top U=I_n\), \(S=S^\top\), and \(D=D^\top\), this becomes
\[
U^\top G + G^\top U - 2S - 2U^\top D U = 0.
\]
Therefore,
\[
S
=
\operatorname{sym}(U^\top G) - U^\top D U
=
B - U^\top D U.
\]

Now impose the row-norm tangent constraint \(\operatorname{diag}(ZU^\top)=0\). Using \(Z=G-US-DU\), this gives
\[
\operatorname{diag}(GU^\top)
-
\operatorname{diag}(USU^\top)
-
\operatorname{diag}(DUU^\top)
=
0.
\]
Since \(\operatorname{diag}(UU^\top)=r\mathbf 1\) and \(D=\operatorname{diag}(\lambda)\), we have
\[
\operatorname{diag}(DUU^\top)=r\lambda.
\]
Substituting \(S=B-U^\top D U\) into \(\operatorname{diag}(USU^\top)\) gives
\[
\operatorname{diag}(USU^\top)
=
\operatorname{diag}(UBU^\top)
-
\operatorname{diag}(UU^\top D UU^\top).
\]
With \(P=UU^\top\), this becomes
\[
\operatorname{diag}(USU^\top)
=
\operatorname{diag}(UBU^\top)
-
\operatorname{diag}(PDP).
\]
Since \(D=\operatorname{diag}(\lambda)\),
\[
\operatorname{diag}(PDP)
=
(P\odot P)\lambda.
\]
Therefore,
\[
\operatorname{diag}(USU^\top)
=
\operatorname{diag}(UBU^\top)
-
(P\odot P)\lambda.
\]

Substituting this into the row constraint gives,
\[
\operatorname{diag}(GU^\top)
-
\left[
\operatorname{diag}(UBU^\top)
-
(P\odot P)\lambda
\right]
-
r\lambda
=
0.
\]
Rearranging,
\[
(rI - P\odot P)\lambda
=
\operatorname{diag}(GU^\top)
-
\operatorname{diag}(UBU^\top).
\]

After solving this linear system for \(\lambda\), define \(D=\operatorname{diag}(\lambda)\) and \(S=B-U^\top D U\). The tangent projection is then
\[
\boxed{
Z = G - US - DU
}
\]

\section{Proofs from the Main Text}\label{sec:proofs}

\begin{proof}[Proof of Proposition~\ref{claim:tall}]
First, suppose $A \in \mathbb R^{m_A \times n_A}$ is column-orthogonal, so $A^\top A = I_{n_A}$. Then
\[
\|A\|_F^2
=
\operatorname{tr}(A^\top A)
=
\operatorname{tr}(I_{n_A})
=
n_A.
\]
On the other hand,
\[
\|A\|_F^2
=
\sum_{i=1}^{m_A} \|A_{i,:}\|_2^2.
\]
If every row of $A$ had unit norm, then this sum would equal $m_A$. Thus we would have $m_A=n_A$, contradicting $m_A>n_A$. Therefore a tall column-orthogonal matrix cannot have unit row norms.

Now suppose $B \in \mathbb R^{m_B \times n_B}$ is row-orthogonal, so $BB^\top = I_{m_B}$. Then
\[
\|B\|_F^2
=
\operatorname{tr}(BB^\top)
=
\operatorname{tr}(I_{m_B})
=
m_B.
\]
Also,
\[
\|B\|_F^2
=
\sum_{j=1}^{n_B} \|B_{:,j}\|_2^2.
\]
If every column of $B$ had unit norm, then this sum would equal $n_B$. Thus we would have $n_B=m_B$, contradicting $m_B<n_B$. Therefore a wide row-orthogonal matrix cannot have unit column norms.

Finally, suppose $C \in \mathbb R^{d \times d}$ is orthogonal, so $C^\top C=I_d$. Then the columns of $C$ are orthonormal, so every column has unit norm. Since $C$ is square and invertible, $C^\top C=I_d$ also implies $CC^\top=I_d$, so the rows of $C$ are orthonormal as well. Hence every row and every column of $C$ has unit norm.
\end{proof}

\begin{proof}[Proof of Proposition~\ref{claim:stateless}]
Let $O_t \in \mathbb{R}^{m \times n}$. By definition,
\[
r_{t,i}
=
\operatorname{RMS}\!\left((O_t)_{i,:}\right)
=
\left(
\frac{1}{n}
\sum_{j=1}^n (O_t)_{ij}^2
\right)^{1/2}.
\]
Since $r_{t,i} > 0$ for every row, the normalized update
\[
(\hat O_t)_{i,:}
=
\frac{(O_t)_{i,:}}{r_{t,i}}
\]
is well-defined. Therefore, for each row $i$,
\[
\operatorname{RMS}\!\left((\hat O_t)_{i,:}\right)
=
\left(
\frac{1}{n}
\sum_{j=1}^n
\left(
\frac{(O_t)_{ij}}{r_{t,i}}
\right)^2
\right)^{1/2}.
\]
Pulling out the constant factor gives
\[
\operatorname{RMS}\!\left((\hat O_t)_{i,:}\right)
=
\frac{1}{r_{t,i}}
\left(
\frac{1}{n}
\sum_{j=1}^n
(O_t)_{ij}^2
\right)^{1/2}
=
\frac{r_{t,i}}{r_{t,i}}
=
1.
\]
Thus every row of $\hat O_t$ has RMS equal to $1$. Hence
\[
\frac{1}{n}
\sum_{j=1}^n
(\hat O_t)_{ij}^2
=
1
\]
for every $i$, so summing over all rows yields,
\[
\sum_{i=1}^m
\sum_{j=1}^n
(\hat O_t)_{ij}^2
=
mn.
\]
Therefore the global RMS of $\hat O_t$ is
\[
\operatorname{RMS}(\hat O_t)
=
\left(
\frac{1}{mn}
\sum_{i=1}^m
\sum_{j=1}^n
(\hat O_t)_{ij}^2
\right)^{1/2}
=
\left(
\frac{mn}{mn}
\right)^{1/2}
=
1.
\]
\end{proof}

\begin{proof}[Proof of Proposition~\ref{claim:leverage}]
Assume that $M$ has full column rank, so that $\sigma_n(M) > 0$. Let
\[
M = U \Sigma V^\top
\]
be the thin SVD of $M$, where $U \in \mathbb{R}^{m \times n}$ has orthonormal columns,
$\Sigma \in \mathbb{R}^{n \times n}$ is diagonal with singular values
\[
\sigma_1(M) \geq \cdots \geq \sigma_n(M) > 0,
\]
and $V \in \mathbb{R}^{n \times n}$ is orthogonal. Then
\[
\operatorname{polar}(M) = U V^\top.
\]
Therefore, since multiplication by $V^\top$ preserves Euclidean row norms,
\[
\|\operatorname{polar}(M)_{i,:}\|_2
=
\|(U V^\top)_{i,:}\|_2
=
\|U_{i,:} V^\top\|_2
=
\|U_{i,:}\|_2.
\]

On the other hand, the $i$th row of $M$ is
\[
M_{i,:}
=
U_{i,:}\Sigma V^\top.
\]
Again using the orthogonality of $V$, we have
\[
\|M_{i,:}\|_2
=
\|U_{i,:}\Sigma V^\top\|_2
=
\|U_{i,:}\Sigma\|_2.
\]
Since every singular value of $\Sigma$ is at least $\sigma_n(M)$,
\[
\|U_{i,:}\Sigma\|_2
\geq
\sigma_n(M)\|U_{i,:}\|_2.
\]
Hence
\[
\|U_{i,:}\|_2
\leq
\frac{\|M_{i,:}\|_2}{\sigma_n(M)}.
\]
If $\|M_{i,:}\|_2 < \epsilon$, then
\[
\|\operatorname{polar}(M)_{i,:}\|_2
=
\|U_{i,:}\|_2
<
\frac{\epsilon}{\sigma_n(M)}.
\]
Squaring both sides gives
\[
\ell_i(M)
:=
\|\operatorname{polar}(M)_{i,:}\|_2^2
<
\frac{\epsilon^2}{\sigma_n(M)^2}.
\]
Therefore,
\[
\ell_i(M)
\leq
\frac{\epsilon^2}{\sigma_n(M)^2},
\]
up to replacing the strict inequality by a non-strict one.
\end{proof}

\begin{proof}[Proof of Proposition~\ref{claim:constraints}]
Let the singular values of $U$ be
\[
\sigma_1(U) \geq \sigma_2(U) \geq \cdots \geq \sigma_n(U) \geq 0.
\]
Since every row of $U$ has squared Euclidean norm $n/m$, we have
\[
\|U\|_F^2
=
\sum_{i=1}^m \|U_{i,:}\|_2^2
=
\sum_{i=1}^m \frac{n}{m}
=
n.
\]
Equivalently,
\[
\sum_{j=1}^n \sigma_j(U)^2 = n.
\]

First suppose that $\|U\|_2 \leq 1$. Then every singular value satisfies $\sigma_j(U) \leq 1$. Therefore
\[
\sum_{j=1}^n \sigma_j(U)^2 \leq n.
\]
But we already know that this sum is exactly $n$, so equality must hold. Since each
$\sigma_j(U)^2 \leq 1$ and there are $n$ terms whose sum is $n$, we must have
\[
\sigma_j(U)^2 = 1
\qquad
\text{for every } j \in [n].
\]
Hence all singular values of $U$ are equal to $1$. Therefore $U^\top U = I_n$. Conversely, suppose that $U^\top U = I_n$.
Then the columns of $U$ are orthonormal, so every singular value of $U$ is equal to
$1$. In particular, $\|U\|_2 = 1 \leq 1$.
Thus, under the constraint $\|U_{i,:}\|_2^2 = n/m$ for every $i$, we have
\[
\|U\|_2 \leq 1
\qquad\Longleftrightarrow\qquad
U^\top U = I_n.
\]
\end{proof}

\subsection{Proof of Theorem~\ref{thm:aurora_converge}}
\begin{proof}
We first record the tangent and normal spaces of the two constraint manifolds.
For the Stiefel manifold
\[
\mathcal S = \{U \in \mathbb R^{m \times n} : U^\top U = I_n\},
\]
the tangent space at $U_*$ is
\[
T_{\mathcal S}(U_*)
=
\{Z \in \mathbb R^{m \times n} : U_*^\top Z + Z^\top U_* = 0\},
\]
and the normal space is
\[
N_{\mathcal S}(U_*)
=
\{U_* A : A = A^\top\}.
\]
For the row-norm manifold
\[
\mathcal O_r
=
\{U \in \mathbb R^{m \times n} : \|U_{i,:}\|_2^2 = r,\ i=1,\dots,m\},
\]
the tangent space at $U_*$ is
\[
T_{\mathcal O_r}(U_*)
=
\{Z \in \mathbb R^{m \times n} : \langle Z_{i,:}, (U_*)_{i,:}\rangle = 0,\ i=1,\dots,m\},
\]
and the normal space is
\[
N_{\mathcal O_r}(U_*)
=
\{D U_* : D \in \mathbb R^{m \times m} \text{ diagonal}\}.
\]

Let $P_* = U_* U_*^\top$. Since $U_* \in \mathcal M = \mathcal S \cap \mathcal O_r$, we have
\[
(P_*)_{ii} = \|(U_*)_{i,:}\|_2^2 = r
\qquad
\text{for every } i.
\]

We now show that the stated regularity condition implies clean intersection.
Suppose
\[
X \in N_{\mathcal S}(U_*) \cap N_{\mathcal O_r}(U_*).
\]
Then there exist a symmetric matrix $A$ and a diagonal matrix $D = \operatorname{diag}(d)$ such that
\[
X = U_* A = D U_*.
\]
In particular,
\[
(I - P_*) D U_* = 0.
\]
Taking squared Frobenius norms gives
\[
0
=
\|(I - P_*)D U_*\|_F^2.
\]
Expanding,
\[
\|(I - P_*)D U_*\|_F^2
=
\operatorname{Tr}\!\left(U_*^\top D(I-P_*)D U_*\right).
\]
Using $P_* = U_*U_*^\top$, this becomes
\[
\|(I - P_*)D U_*\|_F^2
=
\operatorname{Tr}(D P_* D)
-
\operatorname{Tr}(D P_* D P_*).
\]
Since $(P_*)_{ii}=r$, we have
\[
\operatorname{Tr}(D P_* D)
=
r \sum_{i=1}^m d_i^2,
\]
and
\[
\operatorname{Tr}(D P_* D P_*)
=
\sum_{i,j=1}^m d_i d_j (P_*)_{ij}^2.
\]
Therefore
\[
0
=
d^\top \left(r I_m - P_* \odot P_*\right)d,
\]
where $\odot$ denotes the Hadamard product. Since
\[
r I_m - P_* \odot P_*
\]
is positive semidefinite, this implies
\[
d \in \ker(r I_m - P_* \odot P_*).
\]
By the regularity assumption,
\[
\ker(r I_m - P_* \odot P_*) = \operatorname{span}\{\mathbf 1\}.
\]
Hence $d = c\mathbf 1$ for some scalar $c$, and so $D = c I_m$. Thus
\[
X = D U_* = c U_*.
\]
Therefore the only common normal direction is the radial direction $U_*$,
which corresponds to the redundant constraint
\[
\operatorname{Tr}(U^\top U)
=
\sum_{i=1}^m \|U_{i,:}\|_2^2
=
n.
\]
Consequently, after removing this redundant constraint, the two manifolds
$\mathcal S$ and $\mathcal O_r$ intersect cleanly at $U_*$. Equivalently,
in a neighborhood of $U_*$,
\[
T_{\mathcal M}(U_*)
=
T_{\mathcal S}(U_*) \cap T_{\mathcal O_r}(U_*).
\]

Since $\mathcal S$ and $\mathcal O_r$ are smooth embedded submanifolds and
intersect cleanly at $U_*$, the standard local convergence theorem for
alternating projections between smooth manifolds applies. Hence there exists
a neighborhood $\mathcal U$ of $U_*$ such that, for any $U_0 \in \mathcal U$,
the iterates
\[
U_{k+1/2} = \Pi_{\mathcal O_r}(U_k),
\qquad
U_{k+1} = \Pi_{\mathcal S}(U_{k+1/2})
\]
are well-defined and converge linearly to some point
\[
U_\infty \in \mathcal S \cap \mathcal O_r = \mathcal M.
\]

The projections are well-defined locally because every row of $U_*$ has norm
$\sqrt r > 0$, so the projection onto $\mathcal O_r$ is locally given by
row-wise normalization,
\[
\left(\Pi_{\mathcal O_r}(X)\right)_{i,:}
=
\sqrt r\,\frac{X_{i,:}}{\|X_{i,:}\|_2},
\]
and because matrices sufficiently close to $U_*$ have full column rank, so
the projection onto $\mathcal S$ is locally given by the polar factor,
\[
\Pi_{\mathcal S}(X) = \operatorname{polar}(X).
\]
Therefore the alternating projection sequence is locally well-defined and
converges linearly to a point in $\mathcal M$.
\end{proof}
\end{document}